\documentclass[11pt]{article}

\usepackage[utf8]{inputenc} 
\usepackage[T1]{fontenc}    
\usepackage{url}            
\usepackage{booktabs}       
\usepackage{amsfonts}       
\usepackage{nicefrac}       
\usepackage{microtype}      
\usepackage{xcolor}         
\usepackage{mathtools}
\usepackage{bm}
\usepackage{dsfont}
\usepackage{bbm}
\usepackage{thmtools, thm-restate}

\usepackage{amsthm}
\usepackage{amsmath}

\newtheorem{theorem}{Theorem}[section]
\newtheorem*{theorem*}{Theorem}

\newtheorem{proposition}[theorem]{Proposition}
\newtheorem*{proposition*}{Proposition}
\newtheorem{lemma}[theorem]{Lemma}
\newtheorem*{lemma*}{Lemma}

\newtheorem*{conjecture*}{Conjecture}

\newtheorem*{fact*}{Fact}

\newtheorem*{hypothesis*}{Hypothesis}

\newtheorem*{claim*}{Claim}

\theoremstyle{definition}
\newtheorem{definition}[theorem]{Definition}

\theoremstyle{remark}

\newtheorem*{remark*}{Remark}

\newtheorem*{observation*}{Observation}

\newcommand{\eat}[1]{}

\newcommand{\R}{\mathbb{R}}

\newcommand{\calD}{\mathcal{D}}

\newcommand{\poly}{\mathrm{poly}}

 \newcommand{\set}[1]{\{ #1 \} }

\newcommand{\norm}[1]{\lVert #1 \rVert}

\newcommand{\iprod}[1]{\langle#1\rangle}

\newcommand{\Esymb}{\mathbb{E}}
\newcommand{\Psymb}{\mathbb{P}}

\DeclareMathOperator*{\E}{\Esymb}
\DeclareMathOperator*{\Var}{\text{Var}}
\DeclareMathOperator*{\ProbOp}{\Psymb}

\renewcommand{\Pr}{\ProbOp}

\newcommand{\dt}{\, d{t}}

\newcommand{\tw}{\tilde{w}}

\newcommand{\bS}{\mathbb{S}}

\newcommand{\eps}{\varepsilon}
\renewcommand{\epsilon}{\varepsilon}

\newcommand{\tx}{\widetilde{x}}
\newcommand{\tu}{\tilde{u}}
\newcommand{\tv}{\tilde{v}}
\newcommand{\hw}{\widehat{w}}

\newcommand{\tcDx}{\widetilde{\mathcal{D}}_x}

\newif\ifnotes\notesfalse

\ifnotes
\usepackage{color}
\definecolor{mygrey}{gray}{0.50}
\newcommand{\notename}[2]{{\textcolor{blue}{\footnotesize{\bf (#1:} {#2}{\bf ) }}}}

\else

\newcommand{\notename}[2]{{}}

\fi

\newcommand{\anote}[1]{{\notename{Aravindan}{#1}}}

\usepackage{thmtools}
\usepackage{thm-restate}

\usepackage{fullpage}
\usepackage[colorlinks=true,linkcolor=blue, backref=page]{hyperref}

\usepackage[ruled,vlined]{algorithm2e}

\title{Agnostic Learning of General ReLU Activation Using Gradient Descent}

\author{Pranjal Awasthi\\ \small{Google Research}\\ \small{pranjalawasthi@google.com}  \and Alex Tang\footnotemark[1]\\ \small{Northwestern University}\\ \small{alextang@u.northwestern.edu} \and Aravindan Vijayaraghavan\thanks{The last two authors are supported by the National Science Foundation (NSF) under Grant No.~CCF-1652491 and CCF 1934931. The last author was also funded by a Google Research Scholar award. }
 \\ \small{Northwestern University} \\ \small{aravindv@northwestern.edu}}
\date{}

\begin{document}

\maketitle

\begin{abstract}
We provide a convergence analysis of gradient descent for the problem of agnostically learning a single ReLU function with moderate bias under Gaussian distributions. Unlike prior work that studies the setting of zero bias, we consider the more challenging scenario when the bias of the ReLU function is non-zero. Our main result establishes that starting from random initialization, in a polynomial number of iterations gradient descent outputs, with high probability, a ReLU function that achieves an error that is within a constant factor of the optimal error of the best ReLU function with moderate bias. We also provide finite sample guarantees, and these techniques generalize to a broader class of marginal distributions beyond Gaussians.  
\end{abstract}

\section{Introduction} \label{sec:intro}

\newcommand{\newwb}{(\tw,b_w)} 
\newcommand{\neww}{w} 

\anote{I'm trying to make the notation in intro consistent with the rest of the paper. }

Gradient descent forms the bedrock of modern optimization algorithms for machine learning. Despite a long line of work in understanding and analyzing the gradient descent iterates, there remain several outstanding questions on whether they can provably learn important classes of problems. \anote{was:many questions remain regarding its effectiveness for provably learning important classes of problems.} In this work we study one of the simplest learning problems where the properties of gradient descent are not well understood, namely {\em agnostic learning of a single ReLU function}. 

More formally, let $\tilde{D}$ be a distribution over $\mathbb{R}^d \times \mathbb{R}$. A ReLU function is parameterized by $w=\newwb$ where $\tw \in \mathbb{R}^d$ and $b_w \in \mathbb{R}$. For notational convenience, we will consider the points to be in $\R^{d+1}$ by appending $\tx$ with a fixed coordinate $1$ as $x=(\tx,1)$. Let $D$ be the distribution over $\R^{d+1} \times \R$ induced by $\tilde{D}$. 
We define the loss incurred at $w=(\tw,b_w)$ to be 
$$L(w) = \frac{1}{2}\E_{(\tx,y) \sim \tilde{D}} \Big[(\sigma(\tw^\top \tx + b_w) - y)^2 \Big] = \frac{1}{2}\E_{(x,y) \sim D} \Big[(\sigma(w^\top x) - y)^2 \Big].$$ 
Here $\sigma(x) = \max(x,0)$ is the standard rectified linear unit popularly used in deep learning.   The goal in agnostic learning of a ReLU function (or agnostic ReLU regression) is to design a polynomial time learning algorithm that takes as input i.i.d. samples from $D$ and outputs $w=\newwb$ such that $L(w)$ compares favorably with $OPT$ that is given by $$OPT \coloneqq \min_{w=\newwb \in H} \frac{1}{2}\E_{(x,y) \sim D} [(\sigma(w^\top x) - y)^2].$$ 
Here the hypothesis set that algorithm competes with is $H=\{w=(\tw,b_w):  \|\tw\| \in [\tfrac{1}{C_1}, C_1], |b_w| \leq C_2)\}$, where $C_1, C_2>0$ are absolute constants. This implies that the relative bias $|b_w|/\norm{\tw}_2$ is bounded. (We remark that the assumption of $\norm{\tw}=\Theta(1)$ is for convenience; Appendix~\ref{app:scaling} shows why we can assume this is essentially without loss of generality).  This is a non-trivial and interesting regime; when the bias is too large in magnitude the optimal ReLU function fitting the data is either the constant zero function almost everywhere or a linear function almost everywhere.\anote{changes}

\anote{Rephrased a bit}
This agnostic learning problem has been extensively studied and polynomial time learning algorithms exists for a variety of settings. This includes the noisy teacher setting where $\E[y|x]$ is given by a ReLU function \cite{kakade2011efficient, mukherjee2020study} and the fully agnostic setting where no assumption on $y$ is made \cite{goel2019learning, diakonikolas2020algorithms}. In a recent work \cite{frei2020agnostic} analyzed the properties of gradient descent for the above agnostic learning problem when the bias term is assumed to be zero, i.e., $H = \{w=(\tw,0): \|\tw\| \leq O(1)\}$. The gradient descent based learning algorithm corresponds to the following sequence of updates starting from a suitable initializer $w_0$: $w_{t+1} = w_t - \eta \nabla L(w_t)$. The work of \cite{frei2020agnostic} proved that starting from zero initialization and for distributions where the marginal of $x$ satisfies some mild assumptions \anote{was: not degenerate}, gradient descent iterates produce, in polynomial time, a point $w_T$ such that $L(w_T) = O(\sqrt{OPT})$. 

While the above provides the first non-trivial learning guarantees for gradient descent in the case of agnostic ReLU learning, it suffers from a few key limitations. The result of \cite{frei2020agnostic} only applies in the setting when the distribution has a bounded domain and when the bias terms are zero. When the distribution is not bounded, the error of $O(\sqrt{OPT})$ also includes some dimension-dependent terms; e.g., when the marginal of $\tx$ is a standard Gaussian $\mathcal{N}(0,I_{d \times d})$, it gives a $O(\sqrt{d \cdot OPT})$ error. 
Moreover, there is a natural question of improving the bound of $O(\sqrt{OPT})$ on the error of gradient descent (since the most interesting regime of parameters is when $OPT \ll 1$). This is particularly intriguing given the recent result of \cite{diakonikolas2020algorithms} that shows that, assuming zero bias, gradient descent on a convex surrogate for $L(w)$ achieves $O(OPT)$ error. This raises the question of whether the same holds for gradient descent on $L(w)$ itself. In another recent work, the authors in \cite{vardi2021learning} are able to provide convergence guarantees for gradient descent in the presence of bias terms, but under the strong {\em realizability} assumption, i.e, assuming that $OPT=0$. \anote{added:} To the best of our knowledge, there are no existing guarantees for gradient descent (or for any polynomial time algorithm, in fact) for agnostic learning of a ReLU function with bias. 

\subsection{Our Results}

In this work we significantly improve the state of the art of guarantees for gradient descent for agnostic ReLU regression. In particular, we show that when the marginal of $x$ is a Gaussian, gradient descent on $L(w)$ achieves an error of $O(OPT)$, even under the presence of bias terms! In particular, this answers an open question raised in the work of \cite{frei2020agnostic}. There are no additional dependencies on the dimension. Given the recent statistical query lower bound of \cite{goel2019learning} that rules out an additive guarantee of $OPT + \epsilon$ for agnostic ReLU regression, our result shows that vanilla gradient descent on the target loss already achieves near optimal error guarantees. Below we state our main theorem. For details regarding the proofs please see Section~\ref{sec:overview}.

\begin{theorem}
\label{thm:main-informal}
Let $C_1\ge 1, C_2>0, c_3>0$ be absolute constants. Let $D$ be a distribution over $(\tx,y) \in \mathbb{R}^d \times \mathbb{R}$ where the marginal over $\tx$ is the standard Gaussian $\mathcal{N}(0,I)$. Let $H=\{w=(\tw,b_w): \|\tw\| \in [1/C_1, C_1],  |b_w| \leq C_2\}$, and consider population gradient descent iterates: $w_{t+1} = w_t - \eta \nabla L(w_t)$. For a suitable constant learning rate $\eta$, when starting from $w_0=(\tw_0, 0)$ where $\tw_0$ is randomly initialized from a radially symmetric distribution, with at least constant probability $c_3>0$ one of the iterates $w_T$ of gradient descent after $\text{poly}(d, \frac{1}{\epsilon})$ steps satisfies $L(w_T) = O(OPT) + \epsilon$.
\end{theorem}
\anote{It's constant probability, and not high probability.}

Note that the above guarantee applies to one of the intermediate iterates produced by gradient descent within the first $\text{poly}(d, 1/\eps)$ iterations. This is consistent with other convergence guarantees for gradient descent in non-realizable settings where last iterate guarantees typically do not exist~\cite{frei2020agnostic}. One can always pick the iterate among the first $\text{poly}(d, 1/\epsilon)$ steps that has the smallest loss on an independent sample from the distribution $D$. 

The above theorem proves that gradient descent obtains a bound of $O(OPT)$ when the relative bias of the optimal ReLU function is bounded (recall that $\norm{\tw}_2 = \Theta(1)$ for the optimal classifier without loss of generality from Proposition~\ref{prop:scaling}). 
Note that we do not constrain the gradient updates to remain in the set $H$. This result improves upon the existing state-of-the-art guarantees~\cite{frei2020agnostic} of $O(\sqrt{OPT})$ for gradient descent even when specialized to the case of ReLU activations with no bias. Further this gives the first provable guarantees in the setting with non-zero bias. Please see Section~\ref{sec:overview} for the more formal statement and proof. 

We remark that some of the assumptions in Theorem~\ref{thm:main-informal} are made with a view towards a clearer exposition, and similar guarantees hold in more general settings. For example, as in some of the prior work in this area (e.g., \cite{vardi2021learning}), we prove guarantees for gradient descent on the population loss function $L(w)$, as opposed to the empirical loss function on samples. We defer the analysis for the empirical loss function to Section~\ref{sec:empirical}. We now describe the more general  
Moreover the above Theorem~\ref{thm:main-informal} states the guarantees under Gaussian marginals i.e., the distribution over $\tx$ is a standard Gaussian. This already illustrates the improvements guarantees in a basic and well-studied setting. These techniques extend to a broader class of distributions that we describe next.    



\subsection{Guarantees Beyond Gaussian marginals} \label{sec:generaldistribution}

The above algorithmic result can be generalized to a broader class of marginals than Gaussians, that we call $O(1)$-regular marginals. 

\paragraph{$O(1)$-regular marginals: Assumptions about the marginals over $\tx$}
We make the following assumptions about the marginal distribution $\tcDx$ over $\tx \in \R^d$: there exists absolute constants $\beta_1, \beta'_2,\beta_2, \beta_3, \beta_4, \beta_5>0$ and $\beta_0: \R_+ \to \R_+$, such that  
\begin{enumerate}
    \item[(i)] Approximate isotropicity and bounded fourth moments: for every unit vector $u \in \R^d$, $\E_{\tx \sim \tcDx}[\iprod{u,\tx}^2] \in [1/\beta'_2, \beta_2]$, and $\E_{\tx \sim \tcDx}[\iprod{u,\tx}^4] \le \beta_4$. 
    \item[(ii)] Anti-concentration: there exists an absolute constant $\beta_3>0$ such that for every unit vector $\tu \in \R^d$ and $\delta>0$,
    $$\sup_{t \in \R} \Pr_{\tx \sim \tcDx}\Big[ \iprod{\tu,\tx} \in (t - \delta, t+\delta) \Big] \le \min\{\beta_3 \delta,1\}.$$
    \item[(iii)] Spread out: there exists $\beta_0: \R_+ \to \R_+$ such that $\beta_0(|b_v|)>0$ is a constant when $|b_v|$ is a constant  $$\forall \tv \in \mathbf{S}^{d-1},~~\E_{\tx \sim \tcDx}\Big[\sigma(\tv^\top \tx+b_v)\Big]  \ge \beta_0(|b_v|).$$
    \item[(iv)] 2-D projections: In every $2$-dimensional subspace of $\R^d$ spanned by orthonormal unit vectors $\tu_1, \tu_2 \in \R^d$,  we have a set $G_{\tu_1, \tu_2} \subset \R$ such that , 
    \begin{align}
    \Pr_{\tx \sim \tcDx}[ \tu_2^\top \tx \in G_{\tu_1, \tu_2}] &= 1-o(1), \text{ and }\\
    \forall t \in G_{\tu_1, \tu_2},~~    \E_{\tx \sim \tcDx}\Big[\sigma(\tu_1^\top \tx) ~\big\vert~ \tu_2^\top \tx = t\Big] &\ge \beta_5 \cdot \E_{\tx \sim \tcDx}\big[\sigma(\tu_1^\top \tx) \big].
    \end{align}
    In other words, the conditional expectation of $\sigma(\tu_1^\top \tx)$ is not much smaller after conditioning on the projection  in an orthogonal direction $\tu_2$, for most values of $\tu_2^\top \tx$.  
    Note that for a Gaussian $N(0,I)$, the r.v.s $\tu_1^\top \tx, \tu_2^\top \tx$ are independent, so this condition holds with $\beta_5=1$ and $G_{\tu_1, \tu_2} = \R$. 
\end{enumerate}
We remark that Gaussian distribution $\mathcal{N}(0,I)$ is $O(1)$-regular i.e., all the constants $\beta_1, \beta_2, \beta'_2, \beta_5 =1, \beta_3 \le 2$, and $\beta_0(b_v)= \E_{g \sim N(0,1)}[\sigma(g + b_v) ] >0$ for all $b_v \in (-\infty, \infty)$; in fact $\beta_0$ is an increasing function that is $0$ only at $-\infty$. We also note that assumptions of this flavor have also been used in prior works including \cite{vardi2021learning}, which inspired parts of our analysis. In particular, Vardi et al.~\cite{vardi2021learning} assume a lower-bound on the density for any 2-dimensional marginal; our assumption (4) on the 2-dimensional marginals is qualitatively weaker (it is potentially satisfied by even discrete distributions), and moreover we only need the condition to be satisfied for a large fraction of values of $\tu_2^\top \tx$ (and not all).

We now give the generalized version of our main theorem (see Section~\ref{app:generalizing} for the proof).

\begin{theorem}[Generalization beyond Gaussian marginals]
\label{thm:main-regular-dist}
For any absolute constants $C_1 \ge 1, C_2 > 0$, there exists absolute constants $c_3 > 0, c_\eta>0$ such that the following holds. Let $\widetilde{\mathcal{D}}_x$ be a distribution over $(\tx,y) \in \mathbb{R}^d \times \mathbb{R}$ where the marginal over $\tx$ are regular with constant parameters $\beta_1, \beta_2', \beta_2, \beta_3, \beta_4, \beta_5$, and $\beta_0(b_v)$ as defined above. Let $H=\{ (\widetilde{w},b_w): \|\widetilde{w}\| \in [1/C_1, C_1],  |b_w| \leq C_2\}$, and consider population gradient descent iterates: $w_{t+1} = w_t - \eta \nabla L(w)$. For any $\eps > 0$ and learning rate $\eta = c_\eta d^{-1}$, when starting from $w_0=(\widetilde{w}_0, 0)$ where $\widetilde{w}_0$ is randomly initialized from a radially symmetric distribution, with at least constant probability $c_3 > 0$ one of the iterates $w_T$ of gradient descent after $\text{poly}(d, 1/\eps)$ steps satisfies $L(w_T) = O(OPT) + \epsilon$.
\end{theorem}

\section{Related Work} \label{sec:related}

The agnostic ReLU regression problem that we consider has been studied in a variety of settings. In the {\em realizable} setting or when the noise is stochastic with zero mean, i.e., $\E[y|x]$ is a ReLU function, the learning problem is known as isotonic regression and can be solved efficiently via the GLM-tron algorithm \cite{kakade2011efficient, kalai2009isotron}. In the absence of any assumptions on the distribution of $y|x$, the work of \cite{goel2019learning} provided an efficient algorithm that achieves $O(OPT^{2/3}) + \epsilon$ error under Gaussian and log-concave marginals \anote{added:} in the zero-bias setting. The authors also show that it is hard to achieve an additive bound of $OPT + \epsilon$ via statistical query (SQ) algorithms \cite{kearns1994cryptographic}. For the case of zero bias and any marginal over the unit sphere, the work of \cite{goel2017reliably} provides agnostic learning algorithms for the ReLU regression problem that run in time exponential in $1/\epsilon$ and achieve an error bound of $OPT + \epsilon$.
The recent work of \cite{diakonikolas2020algorithms} improved the upper bound of \cite{goel2019learning} to $O(OPT) + \epsilon$ via designing an efficient algorithm that performs gradient descent on a convex surrogate for the loss $L(w)$; very recently they also obtained near optimal sample complexity with a regularized loss \cite{diakonikolas2022optimalsample}. Note that all of the above works that study the fully agnostic setting consider the setting where the bias terms are not present. 

Recent works of \cite{frei2020agnostic, vardi2021learning} consider analyzing gradient descent for the ReLU regression problem. The work of \cite{frei2020agnostic} provides an $O(\sqrt{OPT})$ guarantee for the case of zero bias and bounded distributions. When considering distributions such as the standard Gaussian $\mathcal{N}(0,I)$ the bound of \cite{frei2020agnostic} incurs a dimension dependent term \anote{added:} of the form $O(\sqrt{d} \cdot \sqrt{OPT})$ in the error bound. The work of \cite{vardi2021learning} provides a tighter analysis that also extends to the case of non-zero bias. However the analysis only applies in the realizable setting, i.e., when $OPT$ is zero. Our main result builds upon the above works and provide direct improvements by providing a dimension independent error bound that applies to the case of non-zero bias as well.

There is also a long line of work analyzing gradient descent for broader settings. The works of \cite{ge2015escaping, ge2017learning, jin2017escape, anandkumar2016efficient} show convergence of gradient descent updates to approximate stationary points in non-convex settings under suitable assumptions on the function being optimized. Another line of work considers the global convergence properties of gradient descent. These works establish that gradient descent on highly overparameterized neural networks converges to the global optimum of the empirical loss over a finite set of data points \cite{allen2019convergence, du2018gradient, jacot2018neural,  Zhong2017relu, chizat2018global, lee2019wide, arora2019fine}. Yet another line of work considers the realizable setting where data is generated from an unknown small depth and width neural network. These works analyze the local convergence properties of gradient descent when starting from a suitably close initial point \cite{bartlett2018gradient, zou2020global}. 

\section{Preliminaries} \label{sec:prelims}
We consider agnostically learning a single ReLU neuron with bias through gradient descent under the supervised learning setting. We assume we are given data $(x,y)$, where $x \in \R^{d+1}$ follows the standard Gaussian distribution $\mathcal{N}(0,I)$ in the first $d$ dimensions and the $d+1$'th dimension being a constant 1. We also assume the labels $y \in \R$ are arbitrarily correlated with $x$ and $\sigma(w^\top x)$. 

Note that throughout the paper, we will use $\widetilde{w}, \widetilde{v}, \widetilde{x}$ to denote the first $d$ dimensions of $w, v, x$ respectively, with the last dimension of $w$ being $b_w \in \R$ (similarly for $b_v \in \R$). Therefore, $w^\top x$ is in fact $\widetilde{w}^\top \widetilde{x} + b_w$.

In the analysis, we will compare the current iterate $w$ to any optimizer of the loss $L(w)$. 
\begin{equation}
    v \coloneqq \arg\min_{w \in H} L(w), \text{ where } L(w) = \dfrac{1}{2} \E_{(x,y) \sim D} \Big [(\sigma(w^\top x) - y)^2 \Big ], 
\end{equation}
and the hypothesis set $H=\{w=(\tw,b_w): \|\tw\|  \in [\tfrac{1}{C_1}, C_1], |b_w| \leq C_2)\}$, where $C_1$ and $C_2$ are absolute constants. This ensures that the relative bias $|b_w|/\norm{\tw}_2$ is bounded. 

As we are in the agnostic setting, there may be no $w$ that achieves zero loss. It will be useful to split the loss function $L(w)$ into two components, one of which captures how well $\sigma(w^\top x)$ performs against $\sigma(v^\top x)$. 
\begin{equation}
F(w) \coloneqq \dfrac{1}{2} \E \Big[ (\sigma(w^\top x) - \sigma(v^\top x))^2 \Big], ~~    \nabla F(w) \coloneqq \E \Big[ (\sigma(w^\top x) - \sigma(v^\top x))\sigma'(w^\top x)x \Big].
\end{equation}
We will often refer to $F(w)$ as the \textit{realizable loss}, since it captures the difference between $w$ and $v$; in the realizable setting $L(w) = F(w)$. Note that $F(v) =0$.

\paragraph{Gradient of the Loss.}

The gradient of $L(w)$ with respect to $w$ is
\begin{equation}
    \nabla L(w) = \E \Big[ (\sigma(w^\top x) - y)\sigma'(w^\top x)x \Big]
\end{equation}

where $\sigma'(\cdot)$ is the derivative of $\sigma(\cdot)$, defined as $\sigma'(z) = \mathbbm{1}\{z \geq 0\}$. Note that the ReLU function $\sigma(z)$   is differentiable everywhere except at $z=0$. Following standard convention in this literature, we define $\sigma'(0)=1$. Note that the exact value of $\sigma'(0)$ will have no effect on our results. 

We can also decompose $\nabla L(w)$ as
\begin{align}
    \nabla L(w) &= \E \Big[ (\sigma(w^\top x) - \sigma(v^\top x))\sigma'(w^\top x)x \Big] + \E \Big[ (\sigma(v^\top x) - y)\sigma'(w^\top x)x \Big]
\end{align}
 Therefore,

\begin{equation}
    \nabla L(w) = \nabla F(w) + \E \Big[ (\sigma(v^\top x) - y)\sigma'(w^\top x)x \Big]
\end{equation}

\paragraph{Gradient Descent.} Finally, our paper focuses on the standard gradient descent algorithm with a fixed learning rate $\eta>0$. We initialize at some point $w_0 \in \R^{d+1}$, and at each iteration $t \in \mathbb{N}$ we have $w_{t+1} = w_t - \eta \nabla F(w_t)$. We do not optimize the iteration count in this paper; hence it will be instructive to think of $\eta$ as a non-negligible parameter (that is at least inverse polynomial for polynomial time guarantees) that can be set to be sufficiently small. 

\paragraph{Simplification.} For sake of exposition we will assume that $\norm{\tv}_2 =1$. However, the same analysis goes through when  $\norm{\tv}_2 \in [1/C_1, C_1]$ as well. Moreover Proposition~\ref{prop:scaling} shows that assuming that $\norm{\tv}_2$ is normalized is without loss of generality. Note that we cannot make such a simplifying assumption about the vectors $w_t = (\tw_t, b_w)$ in the intermediate iterations.


\section{Analysis of Gradient Descent (Proof of Theorem~\ref{thm:main-informal})} \label{sec:overview}


In order to highlight our main contribution conceptually, we simplified the statements of the theorems and lemmas stated in the main body for exposition. Hence, in this section, we shall restate Theorem \ref{thm:main-informal} and the key lemmas formally and present the complete proof for our main theorem. The proof for $O(1)$-regular distributions is in Appendix~\ref{app:generalizing}. The formal statement of the main theorem is as follows

\paragraph{Theorem \ref{thm:main-informal}}
(Formal version of the main theorem)
\textit{For any absolute constants $C_1 \ge 1, C_2 > 0$, there exists absolute constants $c_3 > 0, c_\eta>0$ such that the following holds. Suppose $\widetilde{\mathcal{D}}_x$ be a distribution over $(\widetilde{x},y) \in \R^d \times \R$ where the marginal over $\widetilde{x}$ is the standard Gaussian $\mathcal{N}(0,I)$, $H \coloneqq \{ (\widetilde{w},b_w): \|\widetilde{w}\| \in [1/C_1, C_1],  |b_w| \leq C_2\}$, and consider population gradient descent iterates $w_{t+1} = w_t - \eta \nabla L(w)$, with the initializer  $w_0=(\widetilde{w}_0, 0)$ where $\widetilde{w}_0$ is drawn from the radially symmetric distribution in Section~\ref{sec:init}. For any $\eps > 0$ and learning rate $\eta = c_{\eta} d^{-1}$, with at least constant probability $c_3 > 0$,  one of the iterates $w_T$ of population gradient descent after $\poly(d, 1/\eps)$ steps satisfies $L(w_T) = O(OPT) + \epsilon$.}


Note that without loss of generality, we can assume $\eps \leq O(OPT)$. If we cannot make such assumption (e.g. when $OPT \approx 0$), we can use an upper bound on $OPT$ of $O(\eps)$, and carry out the same analysis. 

Recall that $v=(\tv, b_v) \in \R^{d+1}$ is any minimizer of the loss $L(w)$ i.e., $v \coloneqq \arg\min_{w \in H} L(w)$. Hence $L(v)=OPT$. We will assume in the rest of the analysis that $\| \tv \|_2 =1$ for simplifying our exposition. But this is not necessary; see Proposition~\ref{prop:scaling} for why this is without loss of generality.  

In order to find a $w \in \R^{d+1}$ such that $L(w)$ is comparable to $OPT = L(v)$, we aim to find $w$ such that it is close to $v$, i.e. $\|w - v\|$ is small. Note that approximating $v$ suffices to achieve an error close to $OPT$, since we can upper-bound $L(w)$ as

\begin{equation*}
    L(w) =~ \dfrac{1}{2} \E \Big[ (\sigma(w^\top x) - y)^2 \Big]
    =~ \dfrac{1}{2} \E \Big[ (\sigma(w^\top x) - \sigma(v^\top x) + \sigma(v^\top x) - y)^2 \Big]
\end{equation*}
\begin{equation*}
    \leq~ 2 \cdot \dfrac{1}{2} \E \Big[ (\sigma(w^\top x) - \sigma(v^\top x))^2 \Big] + 2 \cdot \dfrac{1}{2} \E \Big[ (\sigma(v^\top x) - y)^2 \Big]
    =~ 2F(w) + 2OPT
\end{equation*}
through Young's inequality. The realizable portion of the loss, $F(w)$, becomes $O(OPT)$ when $\norm{w-v} \le O(\sqrt{OPT})$ (see Lemma~\ref{lem:f-lipschitz} for a proof), and as a consequence we will get $O(OPT)$ error in total. Hence our goal is to prove that for some iterate $T$, we have $F(w_T) \le O(OPT)$ or $\norm{w_T - v}_2 \le O(\sqrt{OPT})$. 

To formalize our intuition above, we adapt a similar proof strategy used in \cite{frei2020agnostic}. Namely, we argue that when optimizing with respect to the agnostic loss $L(w_t)$, we are always making some non-trivial progress due to a decrease in $\norm{w_t-v}$ and due to a decrease in $F(w_t)$ (which is just the realizable portion of the loss). Moreover, whenever we stop making progress, we will argue that at this point either $\norm{w_t-v} \le O(\sqrt{OPT})$ or $\norm{\nabla F(w_t)} \le O(\sqrt{OPT})$; in both cases, this iterate already achieves an error of $O(OPT)$ due to Lemma~\ref{lem:f-lipschitz} and Lemma~\ref{lem:grad-opt}.

\paragraph{Challenges in arguing progress. } Unlike \cite{frei2020agnostic} which only considers ReLU neurons with zero bias, allowing non-zero bias terms imposes extra technical challenges. For example, the probability measure of $\{w^\top x \geq 0, v^\top x \geq 0\}$ under Gaussian distributions, which is vital to deriving the gain in each gradient descent step, does not have a closed-form expression when bias is present. Furthermore we cannot afford to lose any dimension dependent factors or assume boundedness. Thus, to address these difficulties, more detailed analyses (e.g. Lemma \ref{lem:jointprob}, \ref{lem:inner-product-lb}) are needed to facilitate our argument.

On the other hand, tackling non-zero bias terms requires additional assumptions when initializing $w_0$ as well. In our analysis, we assume that $w_0$ is initialized such that $F(w_0)$ is strictly less than $F(0)$ by a constant amount $\delta>0$ (this is inspired by \cite{vardi2021learning}, however $\delta$ can be inverse-polynomially small in their case). This is guaranteed with constant probability by our choice of initialization in Section~\ref{sec:init}. The high-level intuition behind such assumption is to ensure that gradient descent does not get trapped around a highly non-smooth region (e.g. when $w = 0$) by making it start at somewhere better than it, so that $w$ keeps moving closer to $v$.
Despite the same assumption, it is more challenging to implement it in our case (Lemma~\ref{lem:grad-opt}) because of the agnostic setting. 
This is because \cite{vardi2021learning} heavily relies on the realizability assumption to simplify its analysis. 

We also highlight our improvements on the dependency of the dimension $d$. In previous works, the guarantees of  the algorithm has a dependence on $d$ either explicitly or implicitly. For instance, in \cite{frei2020agnostic} the $O(\sqrt{OPT})$ guarantee for ReLU neurons includes a coefficient in terms of $B_X$ (the upper-bound for $\|x\|$), which for Gaussian inputs is in fact $\sqrt{d}$; or for example in \cite{vardi2021learning}, the gain for each gradient descent iteration $\gamma$ comes with a dependency on $c$ (the upper-bound for $\|x\|$) of $c^{-8}$, which for Gaussian is $d^{-4}$. In contrast, we avoid such dependencies on the dimension $d$ in order to obtain our guarantees.

We first establish two important lemmas we will later utilize in proving progress in each iteration. As stated in the preliminaries, we assume in the rest of the section that $\norm{\tv}_2=1$. The first lemma gives a lower bound on the measure of the region where both $\sigma(v^\top x)$ and $\sigma(w_t^\top x)$ are non-zero. 
Our inductive hypotheses will ensure that this lower bound is a constant (if $|b_v|$ is a constant). 
\begin{restatable}[Lower bound on the measure of the intersection]{lemma}{jointprob}
\label{lem:jointprob}
\label{lem:gen:jointprob}
Suppose the marginal distribution $\tcDx$ over $\tx$ is $O(1)$-regular. There exists an absolute constant $c>0$ such that for all $\delta>0$, if $F(w) \leq F(0) - \delta$ then 
\begin{equation}\label{eq:jointprob}
    \Pr[w^\top x \geq 0, v^\top x \geq 0] \geq \frac{\delta^2}{c\norm{w}_2^4 \norm{v}_2^4} = \frac{\delta^2}{c\norm{w}_2^4 (1+|b_v|^2)^2}.
\end{equation}
\end{restatable}

With Lemma \ref{lem:jointprob}, the following lemma allows us to get an improvement on the realizable portion of the loss function as long as the gradient is non-negligible. We state and prove this lemma for the general case of $O(1)$-regular marginal distributions. 

\begin{restatable}[Improvement from the first order term]{lemma}{innerproductlb}
 \label{lem:inner-product-lb}
\label{lem:gen:inner-product-lb}
Suppose the marginal over $\tx$ is $O(1)$-regular. There exists absolute constants $c_1, c_2 > 0$ such that for any $\delta>0$, if $\norm{v}_2, \norm{w}_2 \le B$ and $F(w) \leq F(0) - \delta$ then 
\begin{equation}
    \langle \nabla F(w), w-v \rangle \geq  \gamma  \|w-v\|^2, \text{ where } \gamma = \frac{c_1 \delta^9}{B^{28}}.
\end{equation}
\end{restatable}
The constants $c_1, c_2$ depend on the constants $\beta_1, \beta'_2, \beta_2, \beta_4$ etc. in the regularity assumption of $\tcDx$. 
We remark that for our setting of parameters $\delta = \Omega(1)$ and $B=O(1)$, and hence we will conclude that $\iprod{\nabla F, w-v} \ge \Omega( \norm{w-v}_2^2)$. 

\begin{proof}
This lemma only concerns the ``realizable portion'' of the loss function $F(w)$. 

Let $u=(\tu, b_u) \in \R^{d+1}$ be the unit vector along $w-v$. We have
\begin{align}
    \iprod{\nabla F(w), w-v} & = \E\Big[ \big(\sigma(w^\top x) - \sigma(v^\top x)\big) \sigma'(w^\top x) (w^\top x - v^\top x)\Big] \nonumber \\
    &= \E\Big[ \big(w^\top x - v^\top x\big)^2 \mathbf{1}[w^\top x \ge 0,  v^\top x \ge 0]\Big] \nonumber\\
    &\qquad + \E\Big[ w^\top x \big(w^\top x - v^\top x \big)  \mathbf{1}[w^\top x \ge 0,  v^\top x < 0] \Big]\nonumber \\
    &\ge  \E\Big[ \big(w^\top x - v^\top x\big)^2 \mathbf{1}[w^\top x \ge 0,  v^\top x \ge 0]\Big] \nonumber\\
    &= \norm{w-v}_2^2 \cdot \E\Big[ (u^\top x)^2  \mathbf{1}[w^\top x \ge 0,  v^\top x \ge 0]\Big] \label{eq:in-lb:temp1}
\end{align}

Let $q \coloneqq c \delta^2 / B^8 $ and $\tau \coloneqq c' \frac{\delta^4}{B^{16} }$ for some sufficiently small absolute constant $c'>0$ that will be chosen later. 
We will now lower bound the contribution from just the samples that achieve a value $(u^\top x)^2 > \tau^2$ using Lemma~\ref{lem:jointprob}, that lower bounds $\Pr[w^\top x \ge 0, v^\top x \ge 0] \ge c \delta^2/ B^8=q$:
\begin{align}
&\E\Big[ (u^\top x)^2  \mathbf{1}[w^\top x \ge 0,  v^\top x \ge 0]\Big]  \nonumber \\
&\ge \tau^2 \cdot \Pr_x \Big[ w^\top x \ge 0,  v^\top x \ge 0, (u^\top x)^2 \ge \tau^2 \Big] \nonumber\\
&\ge \tau^2 \cdot \Big(\Pr_x \big[ w^\top x \ge 0,  v^\top x \ge 0 \big] - \Pr_x\big[ (u^\top x)^2 < \tau^2 \big] \Big) \nonumber\\
&\ge \tau^2 \cdot \Big( q - \Pr_{\tx \sim N(0,I)}[|\tu^\top x+b_u|< \tau] \Big). \label{eq:in-lb:temp2}
\end{align}
Now we just need to upper bound $\Pr_{\tx}[|\tu^\top \tx + b_u| < \tau]$.  Let $\beta = \norm{\tu}_2$. 
If $\beta=\norm{\tu}_2 \ll |b_u|$, then $|b_u|$ is itself large, and $\tu^\top \tx$ is too small in comparison to bring down $|\tu^\top \tx + b_u|$. On the other hand, if $\beta = \norm{\tu}_2$ is not too small, then the anti-concentration (or spread out density) of the distribution $\tcDx$ ensures that $|\tu^\top \tx + b_u|$ is small with very low probability. We now formalize this intuition.  

Suppose $\beta = \norm{\tu}_2  \le \tfrac{1}{4\beta_4}(q/2)^{1/4}$. Then $|b_u|>1/2$, since $\norm{u}_2=1$. Also from our choice, $\tau< 1/4$.  Hence by the bounded fourth moments property of $\tcDx$ and Markov's inequality,
$$ \Pr_{\tx \sim \tcDx}[ |\tu^\top \tx + b_u|< \tau ] \le \Pr_{\tx \sim \tcDx}[ |\tu^\top \tx| > \tfrac{1}{4} ] \le \frac{\E_{\tx \sim \tcDx}[\iprod{\tu,\tx}^4]}{(1/4)^4} \le \beta_4 (4\norm{\tu}_2)^4 \le \frac{q}{2}. $$

On the other hand, if $\beta = \norm{\tu}_2 > \tfrac{1}{4 \beta_4}(q/2)^{1/4}$. Suppose $\hat{u}$ is the unit vector along $\tu$. Then using the fact that $\hat{u}^\top \tx = \tu^\top \tx / \norm{\tu}$ is anti-concentrated by the properties of $\tcDx$. Hence we have for some  constant $\beta_3>0$
\begin{align*}
    \Pr_{\tx \sim \tcDx}\Big[ |\tu^\top \tx + b_u|< \tau \Big] &= \Pr_{\tx \sim \tcDx}\Big[ \hat{u}^\top \tx \in ( \tfrac{b_u-\tau}{\norm{\tu}} - \tfrac{b_u+\tau}{\norm{\tu}}) \Big] \le \frac{\beta_3 \tau}{\norm{\tu}} \le 32\beta_3 \beta_4 \Big(\frac{2}{q}\Big)^{1/4} \tau  < \frac{q}{2},
\end{align*}
from our choice of parameters since $\tau = c' \delta^{5/2}/(B^{10} \beta_3 \beta_4)$ for a sufficiently small $c'>0$. Substituting back in \eqref{eq:in-lb:temp2} and \eqref{eq:in-lb:temp1} we have
$$ \iprod{\nabla F(w), w-v} \ge \tau^2 \cdot \frac{q}{2} \ge c_1 \norm{w-v}_2^2 \cdot \frac{\delta^5 }{B^{20} \beta_3^2 \beta_4^2} \cdot \frac{\delta^2}{B^8} \ge c_1 \norm{w-v}_2^2 \cdot \frac{\delta^9}{B^{28}}.$$
\end{proof}

\subsection{Main proof strategy} \label{sec:strategy}
With these two key lemmas, we are now ready to discuss the proof overview of the main theorem (Theorem~\ref{thm:main-informal}). 
We inductively maintain two invariants in every iteration of the algorithm:
\[ \text{(A)}~~~ \norm{w_t -v}_2 \le O(1),~~~~\text{ and } ~~~\text{(B)}~~ F(0) - F(w_t) = \Omega(1). \label{eq:invariants}\]
These two invariants are true at $t=0$ due to our initialization $w_0$. Lemma~\ref{lem:init} guarantees with at least constant probability $\Omega(1)$, both the invariants hold for $w_0$. 
The proof that both the invariants continue to hold follows from the progress made by the algorithm due to a decrease in both $\norm{w_t-v}_2$ and $F(w_t)$ (note that we only need to show they do not increase to maintain the invariant). 

The argument consists of two parts.
First, assuming $F(w_t) \leq F(0) - \delta$ holds (for some constant $\delta>0$), we establish that whenever $\|w_t-v\|^2 > \gamma OPT$ for some constant $\gamma > 0$, gradient descent always makes progress i.e. $\|w_t - v\|^2 - \|w_{t+1} - v\|^2$ is lower bounded. 
Next, we argue that if $w_0$ is initialized such that $F(w_0) \leq F(0) - \delta$ for some constant $\delta > 0$, then throughout gradient descent $F(w_t)$ always decreases, i.e. the inequality $F(w_t) \leq F(w_0) \leq F(0) - \delta$ always holds.

However, unlike \cite{vardi2021learning} where they focus on the realizable setting, analyzing gradient descent on the agnostic loss $L(w)$ is more challenging, since the update depends on $\nabla L(w)$ and not $\nabla F(w)$. In fact, the additional term from the ``non-realizable'' portion of the loss $L(w)$ can overwhelm the contribution from the realizable loss when either $\norm{\nabla F}_2 \le O(\sqrt{OPT})$ or $\norm{w_t - v}_2 \le O(\sqrt{OPT})$. The following two lemmas argue that in both of these cases, the current iterate already achieves $O(OPT)$ error (and this iterate will be the $T$ that satisfies the guarantee of Theorem~\ref{thm:main-informal}). 

\begin{restatable}[Success if $\norm{\nabla F} \le O(\sqrt{OPT})$]{lemma}{gradopt}
\label{lem:grad-opt}
Suppose $B , \delta>0$ are constants such that $\norm{v}_2, \norm{w}_2 \le B$ and $F(w) \leq F(0) - \delta$. Then there exists a constant $C_G>0$, such that if $\|\nabla F(w)\| \leq C_G \sqrt{OPT}$ then $\|w-v\|_2 \leq O(\sqrt{OPT})$.
\end{restatable}
\begin{proof}
We can first apply Lemma \ref{lem:inner-product-lb} to conclude that
\begin{equation*}
    \langle \nabla F(w), w - v \rangle \geq~ \gamma \|w - v\|^2,
\end{equation*}
for some constant $\gamma>0$ (since $B, \delta>0$ are constants). Hence
\begin{equation*}
    \|\nabla F(w)\| \|w - v\| \geq~ \langle \nabla F(w), w - v \rangle \geq~ \gamma \|w - v\|^2, 
\end{equation*}
Thus $\norm{w-v}_2  = O(\sqrt{OPT})$
which implies the lemma.
\end{proof}

We now argue that if $\|w_t - v\| \leq O(\sqrt{OPT})$, then $F(w_t) \leq O(OPT)$ through the following lemma; this is stated and proven for $O(1)$-regular distributions.

\begin{restatable}[Small $\|w_t - v\|$ implies small $F(w_t)$]{lemma}{flipschitz}
\label{lem:f-lipschitz}
\label{lem:f-lipschitz-regular}
Assume $\widetilde{\mathcal{D}}_x$ is $O(1)$-regular with parameters defined above. If $\|w_t - v\|_2 \leq O(\sqrt{OPT + \eps})$ for some $\eps > 0$, then $F(w_t) \leq O(OPT + \eps)$.
\end{restatable}
\begin{proof}
Since ReLU function is 1-Lipschitz (i.e. $|\sigma(z) - \sigma(z')| \leq |z - z'|$),
\begin{equation*}
    F(w_t) = \dfrac{1}{2} \E \Big[ (\sigma(w_t^\top x) - \sigma(v^\top x))^2 \Big]
    \leq~ \dfrac{1}{2} \E \Big[ (w_t^\top x - v^\top x)^2 \Big]
    =~ \dfrac{\|w_t - v\|^2}{2} \E \Big[ (u^\top x)^2 \Big]
\end{equation*}

where we defined $u = \frac{w_t - v}{\|w_t - v\|}$, hence the last equation. 
Now, notice by using Young's inequality, we get
\begin{equation*}
    \E \Big[ (u^\top x)^2 \Big] =~ \E \Big[ (\widetilde{u}^\top \widetilde{x} + b_u)^2 \Big]
    \leq~ 2\E \Big[ (\widetilde{u}^\top \widetilde{x})^2 \Big] + 2b_u^2
    \leq~ 2\beta_2 + 2b_u^2 \leq~ O(1)
\end{equation*}
due to the regularity assumption on $\widetilde{\mathcal{D}}_x$. Hence
\begin{equation*}
    F(w_t) \leq~ \dfrac{\|w_t - v\|^2}{2} \cdot O(1) 
    \leq~ O(\|w_t - v\|_2^2) 
    \leq~ O(OPT + \eps)
\end{equation*}
which concludes the proof.
\end{proof}


\paragraph{Proving progress in $\norm{w_t -v}$ and $F(w_t)$.}

To show $\norm{w_t-v}$ decreases, we establish the following lemma. 

\begin{restatable}[Decrease in $\norm{w_t-v}_2$]{lemma}{induction}
\label{lem:induction}
Assume at time $t$, $F(w_t) \leq F(0) - \delta$ where $\delta > 0$ is a constant. For constants $C_p, C' > 0$ and $\gamma$ defined as in Lemma \ref{lem:inner-product-lb}, if for some $\eps > 0$ $\|w_t - v\|^2 > \gamma^{-1} C_p^2 (OPT + \eps)$, then $\|w_{t+1} - v\|^2 \leq \|w_t - v\|^2 - \eta C' (OPT + \eps)$.
\end{restatable}

Lemma \ref{lem:induction} is a direct consequence of the following inductive statement: for every $t$, either (a) $\|w_t - v\|^2 - \|w_{t+1} - v\|^2 \geq  C \eta OPT$ is true for some constant $C > 0$ or (b) $\|w_t - v\|^2 \leq O(\gamma^{-1} OPT)$ holds. Observe that when (b) holds Lemma~\ref{lem:f-lipschitz} implies the loss is $O(OPT)$; hence we need only assume at time $t$ (b) does not hold yet, thus it suffices focusing on showing (a) is true.

Additionally, note at each timestep $t$,
\begin{align}
    w_{t+1} &= w_t - \eta \nabla L(w_t)\\
     w_{t+1} - v &= w_t - v - \eta\nabla L(w_t)\\
    \Longrightarrow~ \|w_t - v\|^2 - \|w_{t+1} - v\|^2 &= 2\eta \langle \nabla L(w_t), w_t - v \rangle - \eta^2 \|\nabla L(w_t)\|^2.
\end{align}

Therefore, to lower-bound $\|w_t - v\|^2 - \|w_{t+1} - v\|^2$, we will give a lower bound for $\langle \nabla L(w_t), w_t - v \rangle$ and an upper bound for $\|\nabla L(w_t)\|^2$.

To show that $F(w_t)$ decreases we show that at time $t$, if gradient descent continues to make progress towards $v$, then $F(w_{t+1}) \leq F(w_t) \le F(0) - \delta$. 

We can use as a black box two lemmas given in \cite{vardi2021learning} that uses the smoothness of the function to upper bound the contribution from the second order term. 


\begin{restatable}[Lemma D.4 in \cite{vardi2021learning}]{lemma}{smooth}
\label{lem:smooth}
For any $w, w' \in \R^{d+1}$, if $\widetilde{x}$ follows a $O(1)$-regular distribution and $~\forall \lambda \in [0,1]$ there exists constants $C_\ell, C_u > 0$ such that $\|(1-\lambda)w + \lambda w'\| \in [C_\ell, C_u]$, then $\| \nabla F(w) - \nabla F(w') \| \leq (c_1' + \dfrac{8\beta_3(C_u + \sqrt{C_1^2 + C_2^2})c_2'}{C_\ell}) \cdot \|w-w'\|$ where $c_1', c_2' > 0$ are absolute constants.
\end{restatable}

Note that the original Lemma D.4 in \cite{vardi2021learning} assumes the distribution of $\widetilde{x}$ is compactly supported. We hereby provide a dimension-free bound that generalizes the lemma statement to $O(1)$-regular distributions. 

\begin{proof}
Similar to the argument in \cite{vardi2021learning}, we write $\|\nabla F(w) - \nabla F(w')\|$ as
\begin{align*}
    \|\nabla F(w) - \nabla F(w')\| &= \Big \|\E \Big[ (\sigma(w^\top x) - \sigma(v^\top x)) \sigma'(w^\top x) x \Big] - \E \Big[ (\sigma(w'^\top x) - \sigma(v^\top x)) \sigma'(w'^\top x) x \Big] \Big \| \\
    &\leq \Big \| \E \Big [ \mathbbm{1}\{w^\top x \geq 0, w'^\top x \geq 0\} ((w - w')^\top x) x \Big ] \Big \| \\
    &+ \Big \| \E \Big [ \mathbbm{1}\{w^\top x \geq 0, w'^\top x < 0\} (w^\top x - \sigma(v^\top x)) x \Big ] \Big \| \\
    &+ \Big \| \E \Big [ \mathbbm{1}\{w^\top x < 0, w'^\top x \geq 0\} (w'^\top x - \sigma(v^\top x)) x \Big ] \Big \| \\
    &\leq \E \Big [ \| ((w - w')^\top x) x \| \Big ] \\
    &+ \E \Big [ \mathbbm{1}\{w^\top x \geq 0, w'^\top x < 0\} \| (w^\top x - \sigma(v^\top x)) x \| \Big ] \\
    &+ \E \Big [ \mathbbm{1}\{w^\top x < 0, w'^\top x \geq 0\} \| (w'^\top x - \sigma(v^\top x)) x \| \Big ] \\
\end{align*}

Note that we can bound the above three terms similarly as \cite{vardi2021learning} by conditioning on the event in which $\|x\|$ is given, where $f_{\|x\|}$ is the p.d.f. of $\|x\|$.
\begin{align*}
    &= \int \E \Big [ \| ((w - w')^\top x) x \| ~|~ \|x\| \Big ] f_{\|x\|} dx \\
    &+ \int \E \Big [ \mathbbm{1}\{w^\top x \geq 0, w'^\top x < 0\} \| (w^\top x - \sigma(v^\top x)) x \| ~|~ \|x\| \Big ] f_{\|x\|} dx \\
    &+ \int \E \Big [ \mathbbm{1}\{w^\top x < 0, w'^\top x \geq 0\} \| (w'^\top x - \sigma(v^\top x)) x \| ~|~ \|x\| \Big ] f_{\|x\|} dx \\
    &\leq \|w - w'\| \int \|x\|^2 f_{\|x\|} dx \\
    &+ (\|w\| + \|v\|) \int \E \Big [ \mathbbm{1}\{w^\top x \geq 0, w'^\top x < 0\} ~|~ \|x\| \Big ] \|x\|^2 f_{\|x\|} dx \\
    &+ (\|w'\| + \|v\|) \int \E \Big [ \mathbbm{1}\{w^\top x < 0, w'^\top x \geq 0\} ~|~ \|x\| \Big ] \|x\|^2 f_{\|x\|} dx \\
\end{align*}

We can directly bound both $\mathbb{P}\{w^\top x \geq 0, w'^\top x < 0 ~|~ \|x\|\}$ and $\mathbb{P}\{w^\top x <0, w'^\top x \geq 0 ~|~ \|x\|\}$ by $4\beta_3/C_\ell \cdot \|w - w'\| \cdot \|x\|$ using the same argument in the proof of Lemma D.4 of \cite{vardi2021learning}, hence the above can be bounded as
\begin{align*}
    &\leq \|w - w'\| \int \|x\|^2 f_{\|x\|} dx + 2(C_u + \sqrt{C_1^2 + C_2^2}) \cdot \dfrac{4\beta_3}{C_\ell} \|w - w'\| \int \|x\|^3 f_{\|x\|} dx \\
    &\leq (c_1' + \dfrac{8\beta_3(C_u + \sqrt{C_1^2 + C_2^2})c_2'}{C_\ell}) \cdot \|w - w'\|
\end{align*}

for absolute constants $c_1', c_2'$ as the second and third moments of $\|x\|$ due to properties of $O(1)$-regular distributions.
\end{proof}

\begin{restatable}[Lemma D.5 in \cite{vardi2021learning}]{lemma}{smoothexpansion}
\label{lem:smooth-expansion}
Let $f: \R^{d+1} \to \R$ and $\ell > 0$. Assume for any $w, w' \in \R^{d+1}$ such that $\forall~ \lambda \in [0,1]$

\[
    \|\nabla f((1-\lambda)w + \lambda w') - \nabla f(w)\| \leq \ell\lambda \|w' - w\|
\]

then the following holds:
\[
    f(w') \leq f(w) + \langle \nabla f(w), w' - w \rangle + \dfrac{\ell}{2}\|w' - w\|^2
\]

\end{restatable}

These two lemmas are used to get a non-trivial bound on the second order term. The progress in $F(w)$ follows crucially relies on Lemma~\ref{lem:inner-product-lb} that was proven earlier. 



\begin{proof}[Proof of Lemma~\ref{lem:induction}]




Note at each timestep $t$,
\begin{equation}
    w_{t+1} = w_t - \eta \nabla L(w_t)
\end{equation}
\begin{equation}
    w_{t+1} - v = w_t - v - \eta\nabla L(w_t)
\end{equation}
\begin{equation}
    \Longrightarrow~ \|w_t - v\|^2 - \|w_{t+1} - v\|^2 = 2\eta \langle \nabla L(w_t), w_t - v \rangle - \eta^2 \|\nabla L(w_t)\|^2
\end{equation}

therefore to lower-bound $\|w_t - v\|^2 - \|w_{t+1} - v\|^2$, we will give a lower bound for $\langle \nabla L(w_t), w_t - v \rangle$ and an upper bound for $\|\nabla L(w_t)\|^2$.

\paragraph{Lower bounding $\langle \nabla L(w_t), w_t - v \rangle$:} Recall 
that $\nabla L(w_t) = \nabla F(w_t) + \E [(\sigma(v^\top x) - y)\sigma'(w^\top_t x)x]$, implying $\langle \nabla L(w_t), w_t -v \rangle = \langle \nabla F(w_t), w_t - v \rangle + \langle \E [(\sigma(v^\top x) - y)\sigma'(w^\top_t x)x], w_t -v \rangle$.

Since a direct application of Lemma \ref{lem:inner-product-lb} already gives a lower bound on $\langle \nabla F(w_t), w_t - v \rangle$, we need only focus on upper bounding $\big|\langle \E [(\sigma(v^\top x) - y)\sigma'(w^\top_t x)x], w_t -v \rangle \big|$. By Cauchy–Schwarz and Young's inequality, we  get
\begin{align*}
    &\langle \E [(\sigma(v^\top x) - y)\sigma'(w^\top_t x)x], w_t -v \rangle
    = \E [(\sigma(v^\top x) - y)\sigma'(w^\top x)(w^\top_t x - v^\top x)]\\
    &\geq -\sqrt{\E [(\sigma(v^\top x) - y)^2]} \cdot \sqrt{\E [(w^\top_t x - v^\top x)^2\sigma'(w^\top_t x)]}\\
    &\geq -\sqrt{2OPT} \sqrt{\E [((w_t - v)^\top x)^2]}\\
    &\geq -\sqrt{2OPT} \cdot \sqrt{2}\|w_t - v\|
    = -2 \cdot \sqrt{OPT} \cdot \|w_t - v\|
\end{align*}

Putting these bounds together we get
\begin{equation*}
    \nabla L(w_t) = \nabla F(w_t) + \E [(\sigma(v^\top x) - y)\sigma'(w^\top_t x)x] \ge \gamma \norm{w_t - v}_2^2 - 2 \sqrt{OPT} \cdot \norm{w_t-v}_2
\end{equation*}

\paragraph{Upper bounding $\|\nabla L(w_t)\|^2$:} Define
\begin{equation*}
    H(w_t) = \E[\sigma(v^\top x - y)\sigma'(w^\top x)x]
\end{equation*}
and observe that 
\begin{equation*}
    \nabla L(w_t) = \nabla F(w_t) + H(w_t)
\end{equation*}

For the first term,
\begin{equation*}
    \| \nabla F(w_t) \| \leq~ \E[|\sigma(w_t^\top x) - \sigma(v^\top x)| \cdot |\sigma'(w_t^\top x)| \cdot \|x\|] 
    \leq~ \E[|w_t^\top x - v^\top x| \cdot \|x\|]
\end{equation*}

since $\sigma(\cdot)$ is 1-Lipschitz (i.e. $|\sigma(z) - \sigma(z')| \leq |z - z'|$) and $\sigma'(\cdot) \leq 1$. Hence, applying Cauchy-Schwarz yields
\begin{equation*}
    \leq~ \sqrt{\E[|w_t^\top x - v^\top x|^2] \cdot \E[\|x\|^2]}
    \leq~ \|w_t - v\| \cdot \sqrt{d+1}
\end{equation*}

Similarly, for the second term,
\begin{equation*}
    \| H(w_t) \| \leq ~ \E[|\sigma(v^\top x) - y| \cdot \|x\|] 
    \leq~ \sqrt{\E[|\sigma(v^\top x) - y|^2] \cdot \E[\|x\|^2]}
    \leq ~ \sqrt{2 OPT} \cdot \sqrt{d+1}
\end{equation*}

Using the above two expression, we can hence bound $\|\nabla L(w_t)\|^2$ as
\begin{equation}\label{eq:app:grad-l-norm-ub}
    \| \nabla L(w_t) \|^2 \leq~ 2\| F(w_t) \|^2 + 2\| H(w_t) \|^2 \leq~ 4d\|w_t - v\|^2 + 4d OPT
\end{equation}

\paragraph{Lower bounding $\|w_t - v\|^2 - \|w_{t+1} - v\|^2$:} The above inequalities yield
\begin{align*}
    \|w_t - v\|^2 - \|w_{t+1} - v\|^2 &= 2\eta \langle \nabla L(w_t), w_t - v \rangle - \eta^2 \|\nabla L(w_t)\|^2\\
    &\geq 2\eta \cdot \Big [ \gamma \|w_t - v\|^2 - 2\sqrt{OPT}\|w_t-v\| \Big ]
    - 4d \eta^2 \cdot (\|w_t - v\|^2 + OPT)\\
    &\geq 2\eta \cdot \Big [ \gamma \|w_t - v\|^2 - 2\gamma^{1/2} C_p^{-1}\|w_t-v\|^2 \Big ]
    - 4d \eta^2 \cdot (\|w_t - v\|^2 + OPT)\\
    &= 2\eta \Big( \gamma - 2\gamma^{1/2}C_p^{-1} - 2d\eta \Big) \|w_t - v\|^2 - 2\eta \cdot 2d\eta OPT
\end{align*}
due to our assumption that (b) does not hold yet, i.e. $\|w_t - v\| > C_p \gamma^{-1/2} \sqrt{(OPT + \eps)} > C_p \gamma^{-1/2} \sqrt{OPT}$ with some constant $C_p > 0$, implying $\sqrt{OPT} < \gamma^{1/2} C_p^{-1} \|w_t - v\|$. Consequently, by choosing $\eta = c_\eta d^{-1}$, we get 
\begin{equation*}
    \geq~ 2\eta \Big( C_1 \gamma \|w_t - v\|^2 - C_2 OPT \Big)
    \geq~ 2\eta \Big( C_1 C_p^2 (OPT + \eps) - C_2 OPT \Big)
\end{equation*}
\begin{equation*}
    \geq~ \eta C' (OPT + \eps)
\end{equation*}

where $c_\eta, C_1, C_2, C' > 0$ are constants. Hence the proof follows.
\end{proof}

Before we proceed to the proof of main theorem, we prove Lemma~\ref{lem:gen:jointprob}.
\begin{proof}[Proof of Lemma~\ref{lem:gen:jointprob}]
Recall that $F(x)$ is the realizable loss i.e., the loss compared to the optimal solution $v$. Since $F(w) \le F(0) -\delta$, we have
\begin{align}
    F(0) - \delta &\ge F(w) \coloneqq  \frac{1}{2} \E[(\sigma(w^\top x) - \sigma(v^\top x))^2]\nonumber\\
    &= \frac{1}{2} \E[\sigma(w^\top x)^2] - \E[ \sigma(w^\top x) \sigma(v^\top x)] + \frac{1}{2} \E[\sigma(v^\top x)^2] \nonumber\\
    &\ge - \E[\sigma(w^\top x) \sigma(v^\top x)] + F(0) \nonumber \\
    \text{Hence } \delta &\le \E[\sigma(w^\top x) \sigma(v^\top x)]. \label{eq:jointprob:temp1}
\end{align}
Moreover we can also get an upper bound on $\E[\sigma(w^\top x) \sigma(v^\top x)]$ using Cauchy-Schwartz inequality and repeated applications of Young's inequality.   
\begin{align}
    \E[\sigma(w^\top x) \sigma(v^\top x)] &= \E\Big[\mathbf{1}[w^\top x \ge 0, v^\top x \ge 0  ](w^\top x) (v^\top x) \Big] \nonumber \\
    &\le \sqrt{\Pr[w^\top x \ge 0, v^\top x \ge 0]} \cdot \sqrt{\E[(w^\top x)^2 (v^\top x)^2]} \nonumber \\
    &\le \sqrt{\Pr[w^\top x \ge 0, v^\top x \ge 0]} \cdot \sqrt{2\E[(w^\top x)^4]+ 2\E[(v^\top x)^4]} \nonumber \\
    &\le 8\sqrt{\Pr[w^\top x \ge 0, v^\top x \ge 0]} \cdot \sqrt{\E_{\tx \sim N(0,I)}[(\tw^\top \tx)^4]+b_w^4} \cdot \sqrt{ \E_{\tx \sim N(0,I)}[(\tv^\top \tx)^4]+b_v^4}\nonumber \\
    &\le 8 \sqrt{\Pr[w^\top x \ge 0, v^\top x \ge 0]} \cdot \sqrt{(\beta_4\norm{\tw}_2^4+b_w^4) \cdot (\beta_4\norm{\tv}_2^4+b_v^4)} \nonumber \\
    &\le c' \sqrt{\Pr[w^\top x \ge 0, v^\top x \ge 0]} \cdot \norm{w}_2^2 \norm{v}_2^2. \label{eq:jointprob:temp2}
\end{align}
for some constant $c'>0$, where the last but one line follows from the standard bounds on the fourth-moment of an $O(1)$-regular distribution. Combining \eqref{eq:jointprob:temp1} and \eqref{eq:jointprob:temp2} concludes the lemma. 

\end{proof}

\subsection{Proof of Theorem \ref{thm:main-informal}}

With the above lemmas, we are now ready to prove Theorem \ref{thm:main-informal}.
As described in the previous section, we inductively maintain two invariants in every iteration of the algorithm:
\[ \text{(A)}~~~ \norm{w_t -v}_2 \le O(1),~~~~\text{ and } ~~~\text{(B)}~~ F(0) - F(w_t) = \Omega(1). \label{eq:invariants}\]
These two invariants are true at $t=0$ due to our initialization $w_0$. Lemma~\ref{lem:init} guarantees with at least constant probability $\Omega(1)$, both the invariants hold for $w_0$. 
The proof that both the invariants continue to hold follows from the progress made by the algorithm due to a decrease in both $\norm{w_t-v}_2$ and $F(w_t)$ (note that we only need to show they do not increase to maintain the invariant).


The proof consists of three parts. For the first part, at time $t$, assuming $F(w_t) \leq F(0) - \delta$ holds, then by 
directly applying Lemma \ref{lem:induction}, we conclude that as long as $\|w_t - v\|^2 > C_p^2 \gamma^{-1} (OPT + \eps)$ for some constant $C_p > 0$, with learning rate $\eta = c_\eta d^{-1}$ where $c_\eta > 0$ is a constant, gradient descent always makes progress towards $v$. 

In addition, since whenever $\|w_t - v\|^2 > C_p^2 \gamma^{-1} (OPT + \eps)$, $\|w_t - v\|^2 - \|w_{t+1} - v\|^2$ is lower bounded by $\eta C' (OPT + \eps)$ for some constant $C' > 0$, after $T = \|w_0 - v\|^2 C'^{-1} \eta^{-1} (OPT + \eps)^{-1} \leq O(d (OPT + \eps)^{-1})$ iterations we get $\|w_T - v\|^2 \leq O(OPT) + \eps$, and by Lemma \ref{lem:f-lipschitz} this implies $F(w_T) \leq O(OPT) + \eps$, therefore $L(w_T) \leq O(OPT) + \eps$.

In the second part of the proof, we show that if $w_0$ is initialized such that $F(w_0) \leq F(0) - \delta$ for some $\delta > 0$, then while gradient descent is still iterating, the inequality $F(w_t) \leq F(w_0) \leq F(0) - \delta$ always holds.

By Lemma \ref{lem:smooth} which establishes the smoothness of $\nabla F(w)$ between two iterates $w$ and $w'$, we can apply Lemma \ref{lem:smooth-expansion} as
\begin{equation*}
    F(w') \leq F(w) + \langle \nabla F(w), w' - w \rangle + \dfrac{\ell}{2}\|w' - w\|^2
\end{equation*}

where $\ell = (c_1' + \dfrac{8\beta_3(C_u + \sqrt{C_1^2 + C_2^2})c_2'}{C_\ell})$. Note that the conditions in Lemma \ref{lem:smooth} are met since at every timestep $t$, for some constant $C_{\delta} > 0$ $\|w_t\| \geq \frac{\sqrt{\delta}}{\sqrt{C_{\delta}\|v\|}} = C_{\ell} > 0$ implied by Lemma \ref{lem:jointprob}, and $\|w_t\| \leq \sqrt{C_1^2 + C_2^2} = C_u$ as well by assumption, and also recall that $c_1', c_2'$ are $O(1)$-regular distributional constants.

Now, substitute $w$ with $w_t$ and $w'$ with $w_t - \eta \nabla L(w)$ yields
\begin{equation*}
    F(w_t - \eta \nabla L(w_t)) \leq F(w_t) - \eta \langle \nabla F(w_t), \nabla L(w_t) \rangle + \dfrac{\ell\eta^2}{2} \|\nabla L(w_t)\|^2
\end{equation*}

Note that
\begin{equation*}
    \langle \nabla F(w_t), \nabla L(w_t) \rangle
    =~ \langle \nabla F(w_t), \nabla F(w_t) + H(w_t) \rangle
    =~ \|\nabla F(w_t)\|^2 + \langle \nabla F(w_t), H(w_t) \rangle
\end{equation*}

where $H(w_t) = \E[(\sigma(v^\top x) - y)\sigma'(w_t^\top x)x]$. 

Next, we define $u = \tfrac{\nabla F}{\|\nabla F\|}$. Note that $u \in \R^{d+1}$ is a fixed unit vector (it already involves an expectation over $x$); hence
\begin{align*}
    &|\langle \nabla F(w_t), H(w_t) \rangle| 
    =~ \| \nabla F(w_t) \| \cdot  \Big | \Big \langle \E \Big [(\sigma(v^\top x) - y) \sigma'(w^\top x) x \Big ], u \Big \rangle \Big | \\
    &= \| \nabla F(w_t) \| \Big |\E \Big [(\sigma(v^\top x) - y)\sigma'(w^\top x) u^\top x \Big ] \Big |
    \leq~ \| \nabla F(w_t) \| \cdot \Big |\E \big [(\sigma(v^\top x) - y)u^\top x \big ] \Big | \\
    &\leq \| \nabla F(w_t) \| \E \Big [ \big | \sigma(v^\top x) - y \big | \cdot \big | u^\top x \big | \Big ]
    \leq~ \| \nabla F(w_t) \| \cdot \sqrt{OPT} \cdot \sqrt{\E \Big[ (u^\top x)^2 \Big]}
\end{align*}

Note that
\begin{equation*}
    \E \Big[ (u^\top x)^2 \Big] = \E \Big[ (\widetilde{u}^\top \widetilde{x} + b_u)^2 \Big]
    \leq 2\E \Big [(\widetilde{u}^\top \widetilde{x})^2 \Big ] + 2b_u^2
    = 2 \Big( \|\widetilde{u}\|^2 + b_u^2 \Big) = 2
\end{equation*}

Therefore,
\begin{align*}
    &|\langle \nabla F(w_t), H(w_t) \rangle|
    \leq~ \| \nabla F(w_t) \| \cdot \sqrt{OPT} \cdot \sqrt{\E \Big[ (u^\top x)^2 \Big]}
    \leq~ \| \nabla F(w_t) \| \cdot \sqrt{2 OPT} \\
    &\Longrightarrow \langle \nabla F(w_t), H(w_t) \rangle \geq -\| \nabla F(w_t) \| \cdot \sqrt{2 OPT}
\end{align*}

Plugging this back to the expression for $\langle \nabla F(w_t), \nabla L(w_t) \rangle$ yields
\begin{align*}
    &\langle \nabla F(w_t), \nabla L(w_t) \rangle
    =~ \|\nabla F(w_t)\|^2 + \langle \nabla F(w_t), H(w_t) \rangle
    \geq~ \|\nabla F(w_t)\|^2 - \| \nabla F(w_t) \| \cdot \sqrt{2 OPT} \\
    &=~ \|\nabla F(w_t)\| \Big( \|\nabla F(w_t)\| - \sqrt{2 OPT} \Big)
    \geq~ \|\nabla F(w_t)\| \Big( \|\nabla F(w_t)\| - \sqrt{2 (OPT + \eps)} \Big)
\end{align*}

Since we have assumed that gradient descent is still in progress, implying $\|w_t - v\|$ is not at most $\sqrt{OPT + \eps}$ yet, hence by Lemma \ref{lem:grad-opt} $\|\nabla F(w)\| > C_G\sqrt{OPT + \eps}$ at this point, therefore 
\begin{equation*}
    \langle \nabla F(w_t), \nabla L(w_t) \rangle
    \geq~ \|\nabla F(w_t)\| \Big( \|\nabla F(w_t)\| - \sqrt{2 (OPT + \eps)} \Big)
    \geq~ \Omega(OPT + \eps)
\end{equation*}

and by setting $\eta = c_\eta d^{-1}$ with appropriate constants $c_\eta, C'', C_L > 0$,
\begin{align}\label{eq:app:grad-fl-ip}
    &-\eta \langle \nabla F(w_t), \nabla L(w_t) \rangle + \dfrac{\ell\eta^2}{2} \|\nabla L(w_t)\|^2 
    \leq -\eta C''(OPT + \eps) + \dfrac{\ell\eta^2}{2} \|\nabla L(w)\|^2 \\
    &\leq~ \eta \Big( -C''(OPT + \eps) + \dfrac{\ell\eta}{2} \cdot C_L d \|w_t - v\|^2 \Big )
    \leq~ 0
\end{align}

which implies
\begin{equation*}
    F(w_t-\eta\nabla L(w_t)) \leq F(w_t) \leq ... \leq F(w_0) \leq F(0) - \delta
\end{equation*}

Finally, in the last part of the proof, a direct application of Lemma \ref{lem:init} justifies the assumption that $w_0$ can be initialized such that $F(w_0)$ is less than $F(0)$ by a constant amount with constant probability depending only on $b_v$; and since $|b_v| = O(1)$ by assumption, for absolute constants $c_1, c_2 > 0$, with probability at least $c_2$, $F(w_0) \leq F(0) - c_1^2$, which concludes the proof.

\section{Random Initialization} \label{sec:init}

We now prove the initialization lemma assuming weak conditions on the marginal distribution over $\tx \in \R^d$ which is $\tcDx$ (recall that the standard Gaussian $N(0,I)$ also satisfies all of the properties). 
We will initialize $w = (\tw, b_w)$ with $b_w=0$ and $\tw$ drawn from a spherical symmetric distribution $\calD_w$. $\calD_w$ first picks the length $\rho \in \calD_\rho$, and then sets $\tw=\rho \hat{w}$, where $\hat{w}$ is a uniformly random unit vector. The distribution $\calD_\rho$ can be any distribution that is reasonably spread out -- it just needs to place non-negligible probability in any constant length interval $(a_1 \norm{\tv}_2, a_2 \norm{\tv}_2)$ where $a_2 > a_1 >0 $ are constants. 

As stated in the preliminaries, we assume for simplicity that  $\norm{\tv}_2=1$ (or $\Theta(1)$); this is essentially the same as assuming that we know the length scale of $\norm{\tv}_2$, since we can scale the input by this length $\norm{\tv}_2$ (see Proposition~\ref{prop:scaling}). Please refer to Lemma~\ref{lem:gen:unknowninit} when we do not know the length scale of $\norm{\tv}_2$. 
For convenience, we will set $\calD_\rho$ to be the absolute value of a standard Gaussian $N(0,1)$ (or $N(0,\beta^2)$ with $\beta \in [1,2]$.

\begin{lemma}\label{lem:gen:init}\label{lem:init}
There exists 
$c_1(v), c_2(v), c_3(v)>0$ which only depend on $b_v/\norm{\tv}_2$ (and not on the dimension), and are both absolute constants when $|b_v|/\norm{\tv}_2=O(1)$, such that the following holds. 
When $w=(\tw, b_w=0)$ is drawn according to $\tw = \rho \norm{\tv}_2 \hat{w} \sim \calD_w$ described above (with $\hat{w}$ being a uniformly random unit vector, and $\rho \sim \calD_\rho$ being the absolute value of a normal $N(0,\beta^2)$ with $\beta \in [1,2]$). Then with probability at least $c_2(v)>0$, we have 
\begin{align}
    F(w) &\le F(0) -  c_1(v)^2 \norm{\tv}_2^2, \text{ and } \norm{w - v} \le c_3(v) \norm{\tv}_2.  \label{eq:init:statement}
\end{align}

\end{lemma}
In the above lemma, if $\tcDx$ is a standard Gaussian $N(0,I)$, it suffices to choose
for example $c_1(v)= c_0 \E_{g_1 \sim N(0,1)}[\sigma(g_1 + b_v/\norm{\tv}) ] \Big)= c_0 \cdot \Big( \tfrac{b_v}{\norm{\tv}} \Phi(\tfrac{b_v}{\norm{\tv}})+\tfrac{1}{\sqrt{2\pi}}e^{-b_v^2/2\norm{\tv}^2} \Big)$  
for some universal constants $c_0,c'_0, c''_0>0$.  $c_2(v)$ and $c_3(v)$ are also chosen similarly as constants that only depend on $|b_v|/\norm{\tv}$ and not on any dimension dependent term.
We remark that for random initialization to work, we only need the probability of success $\eta>0$ to be non-negligible (e.g., at least an inverse polynomial). We can always try $O(1/\eta)$ many random initializers, and amplify the success probability to be at least $0.99$.    

\newcommand{\bvh}{\widehat{b}_v}
\newcommand{\vh}{\widehat{v}}
\begin{proof}
For convenience we define $\bvh \coloneqq b_v/\norm{\tv}_2, \vh \coloneqq v/\norm{\tv}_2$, so they are normalized w.r.t. the length of $\tv$. The conditions of the lemma assume that $|\bvh|=O(1)$. 

By definition, the distribution of $\tw \in \R^d$ is spherically symmetric. 
\begin{align*}
    F(w) - F(0) &= \frac{1}{2}\E_{x}\Big[ (\sigma(\tw^\top x) - \sigma(\tv^\top x+b_v))^2 \Big] - \frac{1}{2}\E_{x}\Big[ \sigma(\tv^\top x+b_v))^2 \Big] \\
    &= \frac{1}{2}\E_{x}\Big[ (\sigma(\tw^\top x)^2 \Big] - \E_{x}\Big[ \sigma(\tw^\top x) \sigma(\tv^\top x+b_v)) \Big] \\
    &= \frac{\rho^2 \norm{\tv}_2^2}{2} \E_{x}\Big[ (\sigma(\hw^\top x)^2 \Big] - \rho \norm{\tv}_2^2\E_{x}\Big[ \sigma(\hw^\top x) \sigma(\vh^\top x+\bvh)) \Big],
\end{align*}
where $\tw = \rho \norm{\tv}_2 \hw$ with $\hw$ being the unit vector along $\tw$. For a fixed $\rho \in \R_+$, $\hw$ (and hence $\tw$) is picked along a uniformly random direction i.e., $\hw \sim_U \bS^{d-1}$. Hence for $x \sim \tcDx$
\begin{align}
    \E_{\hw \sim \bS^{d-1}}[F((\rho \hw,0)) - F(0)] &= \frac{\rho^2 \norm{\tv}_2^2}{2} \E_{\hw \sim_U \bS^{d-1}} \E_{x \sim \tcDx}\Big[ (\sigma(\hw^\top x)^2 \Big] \\
    &\qquad -  \rho \norm{\tv}_2^2\E_{\hw \sim_U \bS^{d-1}} \E_{x \sim \tcDx}\Big[ \sigma(\hw^\top x) \sigma(\vh^\top x+\bvh)) \Big] \nonumber \\
   & = \norm{\tv}_2^2 \big(c' \rho^2 - 2 c_3(v) \rho \big) ,\label{eq:init:2}  
\end{align}
where $c'>0$ is a universal constant based on our assumptions about $\tcDx$ ($c'=0.5$ for $x \sim N(0,I)$). 

We now derive an expression for $c_3(v)$, and prove that it is a constant independent of the dimension. 
Let $\hw = z_1 \vh + z_2 w^{\perp}$ where $w^{\perp}$ is some unit vector orthogonal to $\tv$. Note that $z_1, z_2$ are r.v.s that depend only on the choice of the initializer (our rotationally invariant distribution), and not on $\tcDx$. For $\hw\sim_U \bS^{d-1}$, the typical values $\E[z_1^2]=1/d$ and $\E[z_2^2]=1-1/d$; moreover $z_1$ and $z_2$ are symmetric (around $0$), and their signs are independent. 
Let the r.v.s $\xi_1 = \iprod{\tx, \vh}, \xi_2 = \iprod{\tx, w^{\perp}}$ denote the marginal distribution along $\vh, w^{\perp}$. The $\xi_1, \xi_2$ are independent of $z_1, z_2$ (but $\xi_1, \xi_2$ could be dependent);  these also satisfy condition (3) about the 2-dimensional marginals of $\tcDx$ because it is $O(1)$-regular.  

\begin{align*} 
c_3(v) & = \E_{\hw \sim_U \bS^{d-1}} \E_{x \sim \tcDx}\Big[ \sigma(\hw^\top x) \sigma(\vh^\top x+\bvh)) \Big]\\
& = 
\E_{z_1, z_2}  \E_{\xi_1, \xi_2 }  \Big[ \sigma(z_1 \xi_1 + z_2  \xi_2) \sigma( \xi_1+\bvh) \Big]\\
& = 
\E_{z_1, z_2}  \E_{\xi_1, \xi_2 }  \Big[ \sigma(z_1  \xi_1 + z_2  \xi_2) \sigma( \xi_1+\bvh) \Big]. 
\end{align*}
Since $z_1$  is a symmetric r.v.,
\begin{align*}
c_3(v) & = 
\frac{1}{2}\E_{z_1, z_2}  \E_{\xi_1, \xi_2 }  \Big[ \sigma(|z_1  \xi_1| + z_2 \xi_2) \sigma(\xi_1+\bvh) \Big] + \frac{1}{2}  \E_{z_1, z_2}  \E_{\xi_1, \xi_2 }  \Big[ \sigma(- |z_1  \xi_1| + z_2  \xi_2) \sigma(\xi_1+\bvh) \Big] \\
&\ge \frac{1}{2}\E_{z_1, z_2}  \E_{\xi_1, \xi_2 }  \Big[ \sigma(|z_1  \xi_1| + z_2  \xi_2) \sigma(\xi_1 +\bvh) \Big]  \\
&\ge \frac{1}{2}  \E_{\xi_1, \xi_2 } \E_{z_1, z_2} \Big[ \sigma(z_2  \xi_2) \sigma( \xi_1 +\bvh) \Big]  \\
&\ge \frac{\E_{z}|z_2|}{2}\int_{t = -\bvh}^{\infty}  p(\xi_1=t)  \cdot (t+\bvh) \cdot \E_{\xi_2}[\sigma(\xi_2) | \xi_1 =t]\dt ~~(\text{since } z_2 \text{ is independent of } \xi_1, \xi_2 )\\
&\ge \frac{1}{8}\int_{t = - \bvh}^{\infty}  p(\xi_1=t)  \cdot (t+\bvh) \cdot \E_{\xi_2}[\sigma(\xi_2) | \xi_1 =t]\dt ~~(\text{since } z_2 \text{ is independent of } \xi_1, \xi_2 )
\end{align*}
 since $\E[|z_2|] \ge 1/4$ (in fact when $d$ is large, $|z_2| = 1-O(1/\sqrt{d})$ for w.h.p.). 
We now split up the inner integral over $t \in [-\bvh, \infty)$ into two parts depending on whether $\E_{\xi_2}[\sigma(\xi_2) | \xi_1 =t] \ge \beta_5 \E_{\xi_2}[\sigma(\xi_2)]$ is satisfied or not. Let $\text{Bad} \subset [-\bvh,\infty)$ be subset where it is not satisfied. Note that from regularity of $\tcDx$, we have that $\Pr[\text{Bad}] = o(1)$. We only take the contribution from $t \in [-\bvh, \infty) \setminus \text{Bad}$:
\begin{align*}
c_3(v)&\ge \frac{\beta_5}{8}   \Big(\int_{t = -\bvh}^{\infty} p(\xi_1=t)  \cdot (t+\bvh) \cdot  \E_{\xi_2}[\sigma(\xi_2)] \dt - \int_{t = -\bvh}^{\infty} p(\xi_1=t) \mathbf{1}[t \in \text{Bad}]  \cdot (t+\bvh) \cdot \beta_5 \E_{\xi_2}[\sigma(\xi_2)]\dt~ \Big)\\
&\ge  \frac{\beta_5 \E_{\xi_2}[\sigma(\xi_2)]}{8} \Big(\E_{\xi_1}[\sigma(\xi_1+\bvh)] - \sqrt{\Pr[\text{Bad}] \cdot \int_{t \in \text{Bad}} p(\xi_1=t)  \cdot (t+\bvh)^2  \dt } ~~\Big)\\ 
&\ge  \frac{\beta_5 \beta_0(0)}{8} \Big(\E_{\xi_1}[\sigma(\xi_1+\bvh)] - \Pr[\text{Bad}]^{1/2}\cdot \Big(2\int_{t \in \R} p(\xi_1=t)  \cdot (t^2+\bvh^2) \dt \Big)^{1/2}~~~ \Big)\\ 
&\ge  \frac{\beta_5 \beta_0(0)}{8} \Big(\beta_0(|\bvh|) - o(1)\cdot \sqrt{2(\beta_2+\bvh^2)}  \Big)\\ 
&\ge c_1 \beta_0(|\bvh|),
\end{align*}
as required for an absolute constant $c_1>0$. Note that the last line used regularity to say $\beta_5 = \Omega(1)$ and lower bound $\E[\sigma(\vh^\top \tx + \bvh)] \ge \beta_0(|\bvh|)$.

We now prove that the first part \eqref{eq:init:statement} holds with non-negligible probability. 
From \eqref{eq:init:2}, we note that for any $\rho \in \big[ \tfrac{c_3(v)}{2c'}, \tfrac{c_3(v)}{c'} \big]$, we have that 
$$ \E_{\hw \sim_U \bS^{d-1}}[F((\rho \hw,0))] \le F(0) - \norm{\tv}_2^2 \frac{c_3(v)^2}{2c'}. $$
Moreover $\rho$ is distributed as the absolute value of a standard normal with variance in $[1,4]$; so we get that $\rho \in \big( \tfrac{c_3(v)}{2c'}, \tfrac{c_3(v)}{c'} \big)$ with probability at least $c_5(v) \coloneqq \tfrac{1}{2\sqrt{2\pi}c'} \cdot e^{-\Omega( c_3(v)^2) }c_3(v) $, which is constant when $|\bvh|$ is a constant. 

Now we condition on the event that $\rho \in \big[ \tfrac{c_3(v)}{2c'}, \tfrac{c_3(v)}{c'} \big]$. For a fixed $\rho$ in this interval,  let $Z$ be a r.v. that captures the distribution of $F((\rho \norm{\tv} \hw,0)) - F(0)$ as $\hw$ is drawn uniformly from the unit sphere $\bS^{d-1}$. Note that $\E[Z] \le - \norm{\tv}_2^2 c_3(v)^2/2c'$. 
\begin{align*}
    \Var[Z] &\le \E[F((\rho \norm{\tv}_2 \hw,0))^2] \le 16 \E_{x \sim \tcDx} [\norm{\tv}^4 \sigma(\rho \hw^\top x)^4] +  16 \E_{x \sim \tcDx} [\norm{\tv}^4 \sigma(\vh^\top x+\bvh)^4] \\
    &\le  256 \norm{\tv}^4 \Big(2\beta_4 + \bvh^4 \Big).      
\end{align*}
Further for $\lambda =  -\E[Z]/2$, we have from the Cantelli-Chebychev one-sided tail inequality we have for some absolute constant $c_6>0$
\begin{align*}
    \Pr\Big[ Z  \le \E[Z]/2 \Big] & \ge \frac{\E[Z]^2}{4 \Var[Z]+ \E[Z]^2} \ge \min\Big\{ c_6 c_3(v)^2 / (\beta_4+\bvh^4) , \frac{1}{2}  \Big\} \eqqcolon c_6(v), 
\end{align*} 
where $c_6(v)$ is a constant when $\bvh$ is a constant. This allows us to conclude that $F(w) < F(0) - \Omega(\norm{\tv}^2)$ with probability at least $c_5(v) \cdot c_6(v)$ which is a constant when $\bvh$ is a constant.  Finally the $\norm{w-v}_2 \le \norm{w}_2 + \norm{\tv}_2$ is upper bound just because  of our choice of $\rho$ and $\norm{\tv}_2$ being upper bounded by assumption.


\end{proof}

%

\subsection{Random Initialization for Unknown Length Scale} \label{sec:unknownlength}

Lemma~\ref{lem:gen:init} shows that if we guess the correct length scale of $\norm{\tv}_2$ up to a factor of $2$, then the random spherically symmetric initialization in Section~\ref{sec:init} succeeds with constant probability. When we have unknown length scale $\norm{\tv}_2  \in [1/M, M]$, the random initialization can try out the different length scales in geometric progression i.e., the length scale $\tau$ is chosen uniformly at random from $\{2^{-j}: j \in \mathbb{Z}, -\log M \le j \le \log M  \}$.  

\paragraph{Random initialization for unknown length scale}

We will initialize $w = (\tw, b_w)$ with $b_w=0$ and $\tw$ drawn from a spherical symmetric distribution $\calD_w$. The length is chosen from the distribution $\calD_\rho$ so that it has a non-negligible probability in any constant length interval $(a_1 \norm{v}_2, a_2 \norm{v}_2)$ where $a_2 > a_1 >0 $ are constants: our specific choice picks the correct length scale with non-negligible probability, and is reasonably spread out. 

We are given a parameter $M$ such that $\norm{v}_2 \in [2^{-\log M}, 2^{\log M}]$ (note that $M$ can have large dependencies on $d$ and other parameters; our guarantees will be polynomial in $\log M$). A random initializer $w=(\tw,0)$ is drawn from $\calD_{\text{unknown}}(M)$ as follows:

\begin{enumerate}
    \item Pick $j$ uniformly at random from $\big\{-\lceil \log M \rceil, -\lceil \log M \rceil +1, \dots, -1, 0, 1, \dots, \lceil \log M \rceil \big\}$. 
    \item $\rho \in \R_{+}$ is drawn according to $\calD_\rho$ as follows: we first pick\footnote{One can pick many other spread out distributions in place of the absolute value of a Gaussian.} $g \sim N(0,1)$ and set $\rho = 2^j |g|$. 
    \item A uniformly random {\em unit} vector $\hat{w} \in \R^d$ is drawn and we output $\tw = \rho \hat{w}$. The initializer is $(\tw, 0)$.  
\end{enumerate}

We prove the following claim about the random initializer. 

\begin{lemma}\label{lem:gen:unknowninit}
There exists 
$c_1(v), c_2(v), c_3(v)>0$ which only depend on $b_v/\norm{\tv}_2$ (and not on the dimension), and are both absolute constants when $|b_v|/\norm{\tv}_2=O(1)$, such that the following holds. 
When $w=(\tw, b_w=0)$ is drawn according to the distribution $\calD_{\text{unknown}}(M)$ described above for some given $M \ge 1$ satisfying $\norm{v}_2 \in [1/M, M]$. Then with probability at least $c_2(v)/ \log M$, we have 
\begin{align}
    F(w) &\le F(0) -  c_1(v)^2 \norm{\tv}_2^2, \text{ and } \norm{w - v} \le c_3(v) \norm{\tv}_2.  \label{eq:init:statement}
\end{align}

\end{lemma}
\begin{proof}
Since $\norm{\tv}_2 \in [1/M, M]$, the random initialization will pick $j^*$ with probability at least $1/(2\log M)$ such that $\norm{\tv}_2 \in [2^{j^*}, 2^{j^*+1}]$. For this choice of $j^*$, we can apply Lemma~\ref{lem:gen:init} (note that we only need a guess of $\norm{\tv}_2$ up to a factor of $2$) to get the required guarantee. 
\end{proof}

\section{Analysis of Gradient Descent with Finite Samples} \label{sec:empirical}
In this section, we analyze gradient descent when trained on finite number of i.i.d. samples $(x_i, y_i) \sim D$. As in the previous sections, we assume the marginal distribution of $x$ is a standard Gaussian. We will utilize the notations below, which are analogous to those defined with respect to the data distribution

\begin{itemize}
    \item $\widehat{L}(w) = \tfrac{1}{2n} \sum_{i=1}^n (\sigma(w^\top x_i) - y_i)^2$
    \item $\widehat{F}(w) = \tfrac{1}{2n} \sum_{i=1}^n (\sigma(w^\top x_i) - \sigma(v^\top x_i))^2$
    \item $\widehat{H}(w) = \tfrac{1}{n} \sum_{i=1}^n (\sigma(v^\top x_i) - y_i) \sigma'(w^\top x_i) x_i$
\end{itemize}

where $n$ is the number of samples.

Since we can only access $n$ samples, we update the weight through full-batch gradient descent as follows
\begin{equation*}
    w_{t+1} = w_t - \eta \nabla \widehat{L}(w_t)
\end{equation*}

We are now ready to analyze gradient descent on finite samples. We first state the main result established in this section

\begin{theorem}\label{thm:app:sample-complexity-main}
Let $C_1\ge 1, C_2>0, c_3'>0$ be absolute constants. Let $D$ be a distribution over $(\tx,y) \in \mathbb{R}^d \times \mathbb{R}$ where the marginal over $\tx$ is the standard Gaussian $\mathcal{N}(0,I)$ and the distribution of $y$ satisfies $|y| \leq B_Y$ for some $B_Y \ge 1$. Let $H=\{w=(\tw,b_w): \|\tw\| \in [1/C_1, C_1],  |b_w| \leq C_2\}$, and consider empirical gradient descent iterates: $w_{t+1} = w_t - \eta \nabla \widehat{L}(w_t)$. For a suitable constant learning rate $\eta$, when starting from $w_0=(\tw_0, 0)$ where $\tw_0$ is randomly initialized from a radially symmetric distribution, when given $\poly(d, 1/\eps, B_Y)$ i.i.d. samples from the data distribution $D$, with at least constant probability $c_3'>0$ one of the iterates $w_T$ of gradient descent after $\text{poly}(d, \frac{1}{\epsilon})$ steps satisfies $L(w_T) = O(OPT) + 2\epsilon$.
\end{theorem}

We remark that the above theorem also holds under our weaker distributional assumptions in Section~\ref{app:generalizing} with an additional sub-Gaussianity assumption on $\tcDx$, as evident from the proof that follows. In order to prove Theorem \ref{thm:app:sample-complexity-main}, we first introduce the following definitions and lemmas. The following definition is a standard tool for establishing uniform convergence guarantees and is deeply related to the notion of Rademacher Complexity. For further details please refer to \cite{shwartz2014ml}.
\begin{definition}[Representativeness]
Given data samples $S = \{z_1, ..., z_n\} \in \mathcal{Z}^n$ and a function class $\mathcal{F} = \{f: \mathcal{Z} \to \mathbb{R}\}$, the representativeness of $S$ with respect to $\mathcal{F}$ is
\[
\mathrm{Rep}(\mathcal{F}, S) = \sup_{f \in \mathcal{F}} ~ \E[f(z)] - \dfrac{1}{n}\sum_{i=1}^n f(z_i)
\]
\end{definition}

Note that representativeness is a random variable. The following lemma quantifies the convergence property of representativeness with respect to the loss function gradient through analyzing its Rademacher complexity.
\begin{lemma}[Concentration of Representativeness]\label{lem:app:rademacher}
For absolute constants $c_1, c_2, c_3 > 0$, with probability at least $1 - \kappa$, the representativeness of random samples $S = \{(\widetilde{x}_i, y_i)\}_{i=1}^n \sim_{i.i.d.} D$ with respect to the function class $\mathcal{F}_j = \{(\sigma(w^\top x) - y)\sigma'(w^\top x)x_j ~:~ \|w\| \leq C_1\}$, $\mathrm{Rep}(\mathcal{F}_j, S)$ is bounded by
\[
\mathrm{Rep}(\mathcal{F}_j, S) \leq  \frac{c_1d B_Y C \sqrt{d\log (C n)}}{\sqrt{n}}   + \sqrt{\dfrac{c_3 d B_Y^2 \log(4/\kappa)}{n}}
\]
where for all $y_i$, $|y_i| \leq B_Y$.
\end{lemma}
\begin{proof}
    Note that $\E [\mathrm{Rep}(\mathcal{F}_j, S)] \leq 2\E [R(\mathcal{F}_j \circ S)]$, where $R(\mathcal{F}_j \circ S)$ is the Rademacher Complexity of the set $\{\{(\sigma(w^\top x_i) - y_i)\sigma'(w^\top x_i)x_{ij}\}_{i=1}^n : \|w\| \leq C_1\}$ (Lemma 26.2 of \cite{shwartz2014ml}). Hence, combining it with McDiarmid’s inequality for almost-bounded difference functions (see \cite{Kutin02McDiarmid}), with probability at least $1-\kappa$ we get
    \begin{align*}
        \mathrm{Rep}(\mathcal{F}_j, S) 
        &\leq \E[\mathrm{Rep}(\mathcal{F}_j, S)] + \sqrt{\dfrac{c_3 d B_Y^2}{n} \log\big(\dfrac{4}{\kappa}\big)}
        \leq 2\E [R(\mathcal{F}_j \circ S)] + \sqrt{\dfrac{c_3 d B_Y^2}{n} \log\big(\dfrac{4}{\kappa}\big)}
    \end{align*}
    
    For the first term, by definition we have
    \begin{align*}
        R(\mathcal{F}_j \circ S) = \frac{1}{n}\E_s \Big [ \sup_{\|w\| \leq C_1} \sum_{i=1}^n s_i x_{ij} \sigma'(w^\top x_i) (w^\top x_i - y_i) \Big ]
    \end{align*}
    
    where $\{s_i\}_{i=1}^n$ are i.i.d. Rademacher random variables. Given the sample $S$, let $R_S$ be the maximum $\ell_2$ norm of a vector $x_i \in S$. We can further upper bound the above as
    
        \begin{align*}
        R(\mathcal{F}_j \circ S) \leq \frac{1}{n}\E_s \Big [ \sup_{\|w\| \leq C_1} \sum_{i=1}^n s_i x_{ij} \sigma'(w^\top x_i) (w^\top x_i) \Big ] + \frac{1}{n}\E_s \Big [ \sup_{\|w\| \leq C_1} \sum_{i=1}^n s_i x_{ij} \sigma'(w^\top x_i) y_i \Big ]
    \end{align*}
    
    For the second term above we can use Massart's finite class lemma~\cite{shwartz2014ml} and noticing that the sup is only over $O(n^{d+1})$ different hypotheses~(since only sign of $w^\top x_i$ matters, and we can use Sauer-Shelah's lemma with linear classifiers in $d$ dimensions~\cite{shwartz2014ml}), we get that 
    \begin{align*}
        \frac{1}{n}\E_s \Big [ \sup_{\|w\| \leq C_1} \sum_{i=1}^n s_i x_{ij} \sigma'(w^\top x_i) y_i \Big ] \leq O(\frac{1}{\sqrt{n}} R_S B_Y \sqrt{d \log n}).
    \end{align*}
    
    To bound the first term, for an appropriate $\epsilon$ to be chosen later, let $H_\epsilon$ be a minimal $\epsilon$-cover for the set $\{w \in \R^d: \|w\| \leq C\}$. It is well known that $|H_\epsilon| = O(C/\epsilon)^d$~\cite{shwartz2014ml}. For any $w \in \R^d$ such that $\|w\| \leq C$ we will denote by $w_\epsilon$ the closest vector to $w$ (in $\ell_2$ distance) in the set $H_\epsilon$. Then we can write
    \begin{align*}
        R(\mathcal{F}_j \circ S) &\leq \frac{1}{n}\E_s \Big [ \sup_{w \in H_\epsilon} \sum_{i=1}^n s_i x_{ij} \sigma'(w^\top x_i) (w^\top x_i) \Big ]\\
        &+      \frac{1}{n}\E_s \Big [ \sup_{\|w\| \leq C_1} \big(\sum_{i=1}^n s_i x_{ij} \sigma'(w^\top x_i) (w^\top x_i) - \sum_{i=1}^n s_i x_{ij} \sigma'(w_\epsilon^\top x_i) (w_\epsilon^\top x_i) \big) \Big ]
    \end{align*}
    
    Noticing that $|(w^\top-w_\epsilon^\top) \cdot x_i| \leq \epsilon R_S$, and the fact that $|\sigma'(t_1)t_1 - \sigma'(t_2)t_2| \leq |t_1 - t_2|$, we get that
    \begin{align}
    \label{eq:rademacher-1}
                R(\mathcal{F}_j \circ S) &\leq \frac{1}{n}\E_s \Big [ \sup_{w \in H_\epsilon} \sum_{i=1}^n s_i x_{ij} \sigma'(w^\top x_i) (w^\top x_i) \Big ] + O(\epsilon R_S B_Y) + O(\frac{1}{\sqrt{n}} R_S B_Y \sqrt{d \log n}). 
    \end{align}
    
    For the first term above we apply Massart's finite class lemma~\cite{shwartz2014ml} to get that
    \begin{align}
    \label{eq:rademacher-2}
        \frac{1}{n}\E_s \Big [ \sup_{w \in H_\epsilon} \sum_{i=1}^n s_i x_{ij} \sigma'(w^\top x_i) (w^\top x_i) \Big ] &\leq O(\frac{1}{\sqrt{n}} R_S C \sqrt{\log (|H_\epsilon|)}).
    \end{align}

From \eqref{eq:rademacher-1} and \eqref{eq:rademacher-2} we get that
    \begin{align*}
        R(\mathcal{F}_j \circ S) = O(\frac{1}{\sqrt{n}} R_S B_Y C \sqrt{d\log (C n /\epsilon)} + \epsilon R_S B_Y).  
    \end{align*}

Substituting $\epsilon = 1/\sqrt{n}$ above we get that
    \begin{align*}
        R(\mathcal{F}_j \circ S) = O \big(\frac{1}{\sqrt{n}} R_S B_Y C \sqrt{d\log (C n)} \big).  
    \end{align*}
Finally, taking the expectation over $S$ and using standard property of Gaussians we get that
\begin{align*}
            \E[R(\mathcal{F}_j \circ S)] = O \big(\frac{1}{\sqrt{n}} d B_Y C \sqrt{d\log (C n)} \big). 
\end{align*}
    
\end{proof}

With these lemmas, we are now ready to prove Theorem \ref{thm:app:sample-complexity-main}.

\paragraph{Proof of Theorem \ref{thm:app:sample-complexity-main}}
The proof consists of three parts highly identical to that of Theorem \ref{thm:main-informal}, hence we only highlight the main difference. As in the proof of Theorem~\ref{thm:main-informal} we will assume that $\norm{v}_2=1$; note that this is without loss of generality from Proposition~\ref{prop:scaling}.  Also as in the proof of Theorem~\ref{thm:main-informal}, we will only argue at one of the iterates $T$ satisfies the required guarantee (this may not be the last iterate). 

In the first part, we rewrite the update rule as
\begin{align}
    w_{t+1} &= w_t - \eta \nabla L(w_t) + \eta \Big ( \nabla L(w_t) - \nabla \widehat{L}(w_t) \Big) = w_t - \eta \nabla L(w_t) + \eta \zeta_t,\nonumber\\
\text{where } \zeta_t &\coloneqq \nabla L(w_t) - \nabla \widehat{L}(w_t), ~\text{ and
}~
g(\zeta_t) \coloneqq -\eta \langle \nabla L(w_t), \zeta_t \rangle + \langle w_t - v, \zeta_t \rangle + \dfrac{\eta \|\zeta_t\|^2}{2}.
\end{align}
We obtain the improvement in each iteration as
\begin{align*}
    \|w_t - v\|^2 - \|w_{t+1} - v\|^2 = 2\eta \langle \nabla L(w_t), w_t - v \rangle -\eta^2 \|\nabla L(w_t)\|^2 - 2\eta g(\zeta_t)
\end{align*}
Note that $2\eta g(\zeta_t)$ is a random variable that depends on $z$ and can possibly be negative. We will later use a uniform convergence bound in  Lemma~\ref{lem:app:rademacher} to bound both $\norm{\zeta_t}$ and hence $|g(\zeta_t)|$ with high probability for all $w$ that is bounded by a fixed constant. Conditioned on this high probability event (given in Lemma~\ref{lem:app:rademacher}), the rest of the analysis is deterministic. Recall that $\eps>0$ is the parameter denoting the desired error. We will maintain the invariants that when gradient descent is still in progress (or we haven't encountered a time step with our desired guarantees), $|g(\zeta_t)| \le \eps$, and $\norm{w_t - v}$ is bounded by a constant.

Recall that
\begin{align*}
    g(\zeta_t) = -\eta \langle \nabla L(w_t), \zeta_t \rangle + \langle w_t - v, \zeta_t \rangle + \dfrac{\eta \|\zeta_t\|^2}{2}
\end{align*}

By applying the upper bound for $\|\nabla L(w_t)\|$ as in the population argument (see Equation \ref{eq:app:grad-l-norm-ub} in the proof of Theorem~\ref{thm:main-informal}), we get for some constant $C'>0$
\begin{align*}
    |g(\zeta_t)| 
    &\leq \eta \|\nabla L(w_t)\|\|\zeta_t\| + \|w_t - v\|\|\zeta_t\| + \dfrac{\eta \|\zeta_t\|^2}{2} \\
    &\leq \eta \sqrt{C' d (\|w_t - v\|^2 + OPT)} \|\zeta_t\| + \|w_t - v\| \|\zeta_t\| + \dfrac{\eta \|\zeta_t\|^2}{2}
\end{align*}

In addition, since at this point gradient descent is still running, $C' \sqrt{OPT} \leq \|w_t - v\|$, hence with suitable constant $C'' > 0$ we can further write
\begin{align*}
    |g(\zeta_t)| 
    &\leq \eta \sqrt{C_G} d C'' \|w_t - v\| \|\zeta_t\| + \dfrac{\eta \|\zeta_t\|^2}{2}
\end{align*}

Again, while gradient descent is still in progress, our induction argument establishes that $\|w_t - v\|^2 - \|w_{t+1} - v\|^2$ is lower-bounded by a non-negative amount, hence $\|w_{t+1} - v\| \leq \|w_t - v\| \leq ... \leq \|w_0 - v\| = O(1)$ which also establishes that every $\|w_t\|$ is upper-bound by a constant. Let $T$ be the time step until which all of the above properties hold (otherwise we have already encountered an iterate where we get the required guarantee).  Therefore we can conclude that $|g(\zeta_t)| \leq O(\|\zeta_t\|) + \tfrac{\eta \|\zeta_t\|^2}{2},~ \forall t \le T$.

We now proceed to bound the magnitude of $\|\zeta_t\|$. Using Lemma \ref{lem:app:rademacher}, for each coordinate of $\zeta_t$ we sample $\poly(d, 1/\eps, B_Y)$ data points so that with probability $1 - \kappa/(d+1)$ its magnitude is at most $\eps/(d+1)$. We then take the union bound over all $d+1$ coordinates and set $\kappa = 1/d^3$ to conclude that with high probability,  
\begin{equation} \label{eq:empirical:help1}
    \forall t \le T, ~~ \|g(\zeta_t)\| \leq \eps, \text{ and } \norm{\zeta_t}_2 \le \eps/c .
\end{equation}

Now, since we have showed that $\|g(\zeta_t)\|$ remains bounded by $\eps$, identical to our argument in the proof for Theorem \ref{thm:main-informal} except modifying the induction hypothesis to be $\|w_t - v\| > C_p \gamma^{-1/2} \sqrt{OPT + 2\eps}$, we conclude that after $T \leq O(d(OPT + 2\eps)^{-1})$ iterations we get $\|w_t - v\|^2 \leq O(OPT) + 2\eps$, and similarly by Lemma \ref{lem:f-lipschitz} this implies both $F(w_T)$ and $L(w_T)$ are at most $O(OPT) + 2\eps$.

Proceeding to the second part of the proof, we will show that while gradient descent is still running, $F(w_t)$ continues to decrease. We rewrite the expression given in Lemma \ref{lem:smooth-expansion} as
\begin{align*}
    F(w_t - \eta \nabla \widehat{L}(w_t)) \leq F(w_t) - \eta \langle \nabla F(w_t), \nabla L(w_t) \rangle + \ell \eta^2 \|\nabla L(w_t)\|^2 - \eta \langle \nabla F(w_t), \zeta_t \rangle + \ell \eta^2 \|\zeta_t\|^2
\end{align*}

At this point, note that we can still argue that $\|\nabla F(w_t)\| > C_G\sqrt{OPT + 2\eps}$ directly by Lemma \ref{lem:grad-opt}, for some constant $C'''>0$,  we can hence upper-bound the second and third terms by directly applying Equation \ref{eq:app:grad-fl-ip}, yielding
\begin{align*}
    \eta \Big( -C'''(OPT + 2\eps) + \ell \eta C_L d \|w_t - v\|^2 \Big)- \eta \langle \nabla F(w_t), \zeta_t \rangle + \ell \eta^2 \|\zeta_t\|^2 
\end{align*}
\begin{align*}
    \leq \eta \Big( -C'''(OPT + 2\eps) + \ell \eta C_L d \|w_t - v\|^2 + C_F d \|w_t - v\|\|\zeta_t\| + \ell \eta \|\zeta_t\|^2 \Big).
\end{align*}
Therefore by applying the same analysis as in the population case, and using \eqref{eq:empirical:help1} we have that the above upper bound is 
\begin{align*}
    \leq \eta \Big( -C'''(OPT + 2\eps) + \ell \eta C_L d \|w_t - v\|^2 + \eps/c \Big)
    \leq 0
\end{align*}

Hence $F(w_t)$ continues to decrease, hence $F(w_t) \leq F(0) - \delta$.

Finally, by Lemma \ref{lem:init} with constant probability gradient descent starts at a point such that $F(w_0) \leq F(0) - \delta$, hence the proof follows.

\section{Conclusion}
In this paper, we provided a convergence analysis of gradient descent for learning a single neuron with general ReLU activations (with non-zero bias terms) and gave improved guarantees under comparable assumptions also made in previous works.
The results of this work are theoretical in nature, as an attempt to understand the convergence guarantees for learning general ReLU neurons through gradient descent; hence we believe they do not have any adverse societal impact. 
We addressed multiple challenges for analyzing general ReLU activations with non-zero bias terms throughout our analyses that may lead to better understanding of the dynamics of gradient descent when learning ReLU neurons. 
However, our analysis does not apply to modern neural networks that have multiple nodes and layers.
The major open direction is to generalize current performance guarantees for networks of multiple neurons and higher depth.  

\bibliographystyle{alpha}
\bibliography{library.bib}


\newpage

\appendix



\section{Proofs for generalizing beyond Gaussian marginals} \label{app:generalizing}

We now describe the generalization to regular distributions of the necessary lemmas for analyzing gradient descent in Section~\ref{sec:regular:lemmas}. 

\subsection{Generalized Lemmas for Regular distributions}\label{sec:regular:lemmas}

In the following lemma, similar to Lemma \ref{lem:induction}, we argue that throughout gradient descent, $\|w_t - v\|^2$ continues to decrease as long as $\|w_t - v\|$ is not too small.
\begin{lemma}[Decrease in $\norm{w_t - v}$]
\label{lem:induction-regular}
Let $\widetilde{\mathcal{D}}_x$ be $O(1)$-regular with parameters defined above. Assume at time $t$, $F(w_t) \leq F(0) - \delta$ where $\delta > 0$ is a constant. For constants $C_p, C' > 0$ and $\gamma$ defined as in Lemma \ref{lem:inner-product-lb}, if for some $\eps > 0$ $\|w_t - v\|^2 > \gamma^{-1} C_p^2 (OPT + \eps)$, then $\|w_{t+1} - v\|^2 \leq \|w_t - v\|^2 - \eta C' (OPT + \eps)$.
\end{lemma}
\begin{proof}




Resembling the proof of Lemma \ref{lem:induction}, to lower-bound $\|w_t - v\|^2 - \|w_{t+1} - v\|^2$, we will give a lower bound for $\langle \nabla L(w_t), w_t - v \rangle$ and an upper bound for $\|\nabla L(w_t)\|^2$.

\paragraph{Lower bounding $\langle \nabla L(w_t), w_t - v \rangle$} 

A direct application of Lemma \ref{lem:gen:inner-product-lb} already gives a lower bound on $\langle \nabla F(w_t), w_t - v \rangle$, hence we need only focus on lower bounding $\langle \E [(\sigma(v^\top x) - y)\sigma'(w^\top_t x)x], w_t -v \rangle$, and by Cauchy–Schwarz and Young's inequality, we immediately get
\begin{align*}
    &\langle \E [(\sigma(v^\top x) - y)\sigma'(w^\top_t x)x], w_t -v \rangle
    ~=~ \E [(\sigma(v^\top x) - y)\sigma'(w^\top x)(w^\top_t x - v^\top x)] \\
    &\geq -\sqrt{\E [(\sigma(v^\top x) - y)^2]} \sqrt{\E [(w^\top_t x - v^\top x)^2\sigma'(w^\top_t x)]}
    \geq -\sqrt{2OPT} \sqrt{\E [((w_t - v)^\top x)^2]} \\
    &\geq -\sqrt{2OPT} \cdot \sqrt{C_{\beta} \beta_2}\|w_t - v\|
    \geq -C_{\beta}' \sqrt{\beta_2} \cdot \sqrt{OPT} \cdot \|w_t - v\|
\end{align*}

with constants $C_{\beta}, C_{\beta}' > 0$. Putting the bound above along with that of Lemma \ref{lem:gen:inner-product-lb} together we get
\begin{equation*}
    \nabla L(w_t) = \nabla F(w_t) + \E [(\sigma(v^\top x) - y)\sigma'(w^\top_t x)x] \ge \gamma \norm{w_t - v}_2^2 - C_{\beta}' \sqrt{\beta_2} \cdot \sqrt{OPT} \cdot \norm{w_t-v}_2
\end{equation*}

\paragraph{Upper bounding $\|\nabla L(w_t)\|^2$} 


Recall $\nabla L(w_t) = \nabla F(w_t) + H(w_t) ~\Longrightarrow~ \|\nabla L(w_t)\|^2 \leq 2\|\nabla F(w_t)\|^2 + 2\|H(w_t)\|^2$. For the first term,
\begin{equation*}
    \| \nabla F(w_t) \| \leq~ \E[|\sigma(w_t^\top x) - \sigma(v^\top x)| \cdot |\sigma'(w_t^\top x)| \cdot \|x\|] 
    \leq~ \E[|w_t^\top x - v^\top x| \cdot \|x\|]
\end{equation*}

since $\sigma(\cdot)$ is 1-Lipschitz (i.e. $|\sigma(z) - \sigma(z')| \leq |z - z'|$) and $\sigma'(\cdot) \leq 1$. Hence, applying Cauchy-Schwarz yields
\begin{equation*}
    \leq~ \sqrt{\E[|w_t^\top x - v^\top x|^2] \cdot \E[\|x\|^2]}
    \leq~ \|w_t - v\| \cdot \sqrt{\beta_2 d+1}
\end{equation*}

Similarly, for the second term,
\begin{equation*}
    \| H(w_t) \| \leq ~ \E[|\sigma(v^\top x) - y| \cdot \|x\|] 
    \leq~ \sqrt{\E[|\sigma(v^\top x) - y|^2] \cdot \E[\|x\|^2]}
    \leq ~ \sqrt{2 OPT} \cdot \sqrt{\beta_2 d+1}
\end{equation*}

Using the above two expression, we can hence bound $\|\nabla L(w_t)\|^2$ as
\begin{equation*}
    \| \nabla L(w_t) \|^2 \leq~ 2\| F(w_t) \|^2 + 2\| H(w_t) \|^2 \leq~ C_{\beta}'' d\|w-v\|^2 + C_{\beta}'' d OPT
\end{equation*}

for some constant $C_{\beta}'' > 0$.

\paragraph{Lower bounding $\|w_t - v\|^2 - \|w_{t+1} - v\|^2$} The above inequalities yield
\begin{align*}
    &\|w_t - v\|^2 - \|w_{t+1} - v\|^2 = 2\eta \langle \nabla L(w_t), w_t - v \rangle - \eta^2 \|\nabla L(w_t)\|^2 \\
    &\geq 2\eta \cdot \Big [ \gamma \|w_t - v\|^2 - C_{\beta}'\sqrt{\beta_2} \sqrt{OPT}\|w_t-v\| \Big ]
    - C_{\beta}'' d \eta^2 \cdot (\|w-v\|^2 + OPT) \\
    &\geq 2\eta \cdot \Big [ \gamma \|w_t - v\|^2 - C_{\beta}'\sqrt{\beta_2} \gamma^{1/2} C_p^{-1}\|w_t-v\|^2 \Big ]
    - C_{\beta}'' d \eta^2 \cdot (\|w-v\|^2 + OPT) \\
    &= 2\eta \Big( \gamma - C_{\beta}'\sqrt{\beta_2} \gamma^{1/2}C_p^{-1} - \dfrac{C_{\beta}''}{2} d \eta \Big) \|w_t - v\|^2 - 2\eta \cdot \dfrac{C_{\beta}''}{2} d\eta OPT
\end{align*}

due to our assumption that (b) does not hold yet, i.e. $\|w_t - v\| > C_p \gamma^{-1/2} \sqrt{(OPT + \eps)} > C_p \gamma^{-1/2} \sqrt{OPT}$ with some constant $C_p > 0$, implying $\sqrt{OPT} < \gamma^{1/2} C_p^{-1} \|w_t - v\|$. Consequently, by choosing $\eta \leq O(d^{-1})$, we get 
\begin{align*}
    &\geq 2\eta \Big( C_1 \gamma \|w_t - v\|^2 - C_2 OPT \Big)
    \geq 2\eta \Big( C_1 C_p^2 (OPT + \eps) - C_2 OPT \Big) \\
    &\geq \eta C' (OPT + \eps)
\end{align*}

where $C_1, C_2, C' > 0$ are constants. Hence the proof follows.
\end{proof}

With the lemmas above, we are now ready to prove Theorem \ref{thm:main-regular-dist}.

\subsection{Proof of Theorem \ref{thm:main-regular-dist}}

The proof highly resembles that of Theorem \ref{thm:main-informal} by inductively maintaining the same two invariants in every iteration of the algorithm:
\[ 
\text{(A)}~~~ \norm{w_t -v}_2 \le O(1),~~~~\text{ and } ~~~\text{(B)}~~ F(0) - F(w_t) = \Omega(1). \label{eq:invariants}\]
Hence, we only highlight the difference compared to the previous proof.

The proof also consists of three parts. For the first part, we simply replace Lemmas \ref{lem:induction} and \ref{lem:f-lipschitz} with Lemmas \ref{lem:induction-regular} and \ref{lem:f-lipschitz-regular}, resulting in the same argument that after $T \leq O(d (OPT + \eps)^{-1})$ iterations we get $F(w_T) \leq O(OPT) + \eps$, therefore $L(w_T) \leq O(OPT) + \eps$.

In the second part of the proof, Lemmas \ref{lem:grad-opt}, \ref{lem:smooth}, \ref{lem:smooth-expansion} remain valid for $O(1)$-regular distributions, therefore we need only note that for any unit vector $u \in \R^{d+1}$,
\[
    \E [(u^\top x)^2] \leq~ 2 \E [(\widetilde{u}^\top \widetilde{x})^2] + 2b_u^2 \leq 2\beta_2 + 2b_u^2
    \leq O(\beta_2)
\]

which only affects the bounds for $\| \nabla L(w_t) \|$ and $\| H(w_t) \|$ up to a constant factor. Hence the inequality $F(w_t - \eta \nabla L(w_t)) \leq F(w_t) \leq F(w_0) \leq F(0) - \delta$ also holds.

Finally, in the last part of the proof, a direct application of Lemma \ref{lem:gen:init} justifies the initialization assumption, which concludes the proof.


\section{Invariance to Scaling}\label{app:scaling} \label{sec:scaling}


In this section we show that the guarantees of gradient descent do not change by scale the instance by a multiplicative factor of $\alpha$. Here the instance is scaled by only multiplying the $y$ values by the same factor $\alpha$ (but {\em not} scaling the point $x$). This allows us to assume that $\norm{\tv}_2=1$ without loss of generality as long as the initializer is also in the same length scale (see Lemma~\ref{lem:gen:unknowninit} how the random initialization finds the correct length scale with reasonable probability).

Recall that we consider the loss function $L$ which given hypothesis $w=(\tw, b_w) \in \R^{d+1}$ and input distribution $\tilde{\calD}$ over $(\tx,y) \in \R^d \times \R$ is 
\begin{equation} \label{eq:scaling:loss}
    L(w, \tilde{\calD})=\tfrac{1}{2}\E_{(\tx, y) \sim \tilde{\calD}}[(\sigma(\tw^\top \tx + b_w) -y)^2]. 
\end{equation}

We show the following simple proposition. 
\begin{proposition}\label{prop:scaling}
Let $\alpha>0$, and let $\tilde{\calD}$ be any distribution over $(\tx,y) \in \R^{d} \times \R$, let $\tilde{\calD}_\alpha$ be the corresponding distribution given by $(\tx,y' = \alpha y)$ (only the $y$ values are scaled). For every $w=(\tw,b_w) \in \R^{d+1}$ we have that   
\begin{align}
    L(\alpha w, \tilde{\calD}_\alpha) = \alpha^2 \cdot L(w, \tilde{\calD}), ~~\text{ where } \alpha w = (\alpha \tw, \alpha b_w). 
\end{align}
Moreover, for two runs of gradient descent (with the same step size $\eta$) producing iterates $w_0, w_1, \dots, w_T$ when run on $\tilde{\calD}$ and producing iterates $w'_0, w'_1, \dots, w'_T$ when run on $\tilde{\calD}_\alpha$, we have:  
\begin{align}
    \text{if } w'_0=\alpha w_0, ~~\text{ then }~~\forall t \in \set{0,1,2, \dots, T}, ~ w'_t = \alpha \cdot w_t. 
\end{align}
Finally, if $OPT$ and $OPT_\alpha$ are the optimal losses for $\tilde{\calD}$ and $\tilde{\calD}_\alpha$ respectively, then for any $\beta>0$, $F(w_t) \le \beta \cdot OPT$ if and only if $F(w'_t) \le \beta \cdot OPT_\alpha$. 
\end{proposition}
\begin{proof}
The first part follows directly from \eqref{eq:scaling:loss}. We have
\begin{align*}
    L(\alpha w, \tilde{\calD}_\alpha) & = \tfrac{1}{2}\E_{(\tx, y') \sim \tilde{\calD}_\alpha}[(\sigma(\alpha \tw^\top \tx + \alpha b_w) -y')^2] = \tfrac{1}{2}\E_{(\tx, y) \sim \tilde{\calD}}[(\sigma(\alpha \tw^\top \tx + \alpha b_w) -\alpha y)^2] \\
    &= \tfrac{1}{2}\E_{(\tx, y) \sim \tilde{\calD}}[ ( \alpha \sigma(\tw^\top \tx + b_w) - \alpha y)^2] = \alpha^2 L(w, \tilde{\calD}). 
\end{align*}
The second part uses the form of the gradient update through a simple induction. The base case is true since by assumption $w'_0=\alpha w_0$. Suppose $w'_t = \alpha w_t$. Let $\calD_\alpha$ denote the distribution over $x=(\tx,1),y'=\alpha y$ corresponding to $\tilde{\calD}_\alpha$. Recall that $w'_{t+1} = w'_t - \nabla L(w'_t, \calD_\alpha)$ where
\begin{align*}
    \nabla L(w'_t, \calD_\alpha) &= \E_{(x,y') \sim \calD_\alpha} \Big[ (\sigma(w^\top x) - y')\sigma'(w^\top x)x \Big].\\
    \text{Hence } w'_{t+1} &= w'_t - \eta \nabla L(w'_t, \calD_\alpha) = w'_t - \eta \E_{(x,y') \sim \calD_\alpha} \Big[ (\sigma(\alpha w^\top x) - y')\sigma'(\alpha w^\top x)x \Big] \\
    &= \alpha w_t - \alpha  \cdot \eta \E_{(x,y) \sim \calD} \Big[ (\sigma(\alpha w^\top x) - \alpha y)\sigma'(w^\top x)x \Big]\\
    &= \alpha w_t - \alpha \cdot \nabla L(w_t, \calD) = \alpha w_{t+1}. 
\end{align*}
Note that the last but second line used the fact that $\sigma'(\alpha w^\top x) = \mathbf{I}[\alpha w^\top x \ge 0] = \sigma'(w^\top x)$ when $\alpha>0$. The last part of the proposition just follows from the first claim that $L(\alpha w, \tilde{\calD}_\alpha) = \alpha^2 \cdot L(w, \tilde{\calD})$ for all $w$ applied to $w_T, w'_T = \alpha w_T$ and the optimal solutions corresponding to $OPT$ and $OPT_\alpha$.   
\end{proof}

\emph{Remark.} We remark that the above proposition essentially shows that we can assume that $\norm{\tv}_2 =1$, almost without loss of generality. However, this proposition assumes that initializer $\tw_0$ can also be scaled accordingly i.e., the initializer $\tw_0$ continues to have the same length scale as $\tv$. This is achieved by our random initialization strategy in Section~\ref{sec:unknownlength}, since it tries out many different length scales. 

\end{document}